%% file: main_new.tex
\documentclass[twoside]{article}

\usepackage[accepted]{aistats2022}

\usepackage{balance}
\usepackage[normalem]{ulem}

\usepackage[utf8]{inputenc}
\usepackage{amsfonts, dsfont, amssymb, amsthm, amsmath,graphicx,xcolor,algorithm,algpseudocode}
\usepackage{cleveref,subcaption}
\input{tdefs}

\newtheorem{thm}{Theorem}
\newtheorem{cor}{Corollary}
\newtheorem{prop}{Proposition}
\newtheorem{lem}{Lemma}

\newcommand{\bx}{\b{\mu}}
\newcommand{\bxh}{\hat{\bx}}
\newcommand{\bxt}{\tilde{\bx}}
\newcommand{\budget}{T}
\newcommand{\ferr}{\epsilon}
\newcommand{\perr}{\delta}

\newcommand{\ba}{\boldsymbol{\alpha}}
\newcommand{\bs}{\boldsymbol{\sigma}}
\newcommand{\bsh}{\hat{\bs}}

\renewcommand{\epsilon}{\varepsilon}

\begin{document}
\runningtitle{Approximate Function Evaluation via Multi-Armed Bandits}

\twocolumn[

\aistatstitle{Approximate Function Evaluation \\ via Multi-Armed Bandits}

\aistatsauthor{ Tavor Z. Baharav \And  Gary Cheng \And  Mert Pilanci \And   David Tse }

\aistatsaddress{ Stanford University} ]

\begin{abstract}
We study the problem of estimating the value of a known smooth function $f$ at an unknown point $\bx \in \R^n$, where each component $\mu_i$ can be sampled via a noisy oracle.
Sampling more frequently components of $\bx$ corresponding to directions of the function with larger directional derivatives is more sample-efficient.
However, as $\bx$ is unknown, the optimal sampling frequencies are also unknown.
We design an instance-adaptive algorithm that learns to sample according to the importance of each coordinate, and with probability at least $1-\delta$ returns an $\epsilon$ accurate estimate of $f(\bx)$.
We generalize our algorithm to adapt to heteroskedastic noise, and prove asymptotic optimality when $f$ is linear.
We corroborate our theoretical results with numerical experiments, showing the dramatic gains afforded by adaptivity.

\end{abstract}

\section{Introduction}
Estimation is a ubiquitous and expensive task in many modern applications.
Often, however, we are not interested in estimation for estimation's sake, but rather for a downstream task through an application specific function.
For an unknown parameter vector $\bx \in \R^n$ which we are given noisy coordinate-wise query access to, it is well understood how to sample to construct an estimate to minimize the approximation error between our estimate $\bxh$ and the true vector $\bx$.
Instead, if we are interested in estimating $f(\bx)$, a natural question is whether we can leverage knowledge of this function $f$ to construct a more efficient estimation procedure.
In this paper we show that using adaptivity we can dramatically reduce the number of samples required to estimate $f(\bx)$ to a target accuracy $\epsilon$, exploiting the fact that we are interested in $\bx$ solely through its evaluation $f(\bx)$. 
This problem has been well studied in the special case of $f(\bx)= \max_{i \in [n]} \mu_i$ in the multi-armed bandit literature.
In this paper, we focus on settings where $f$ has Lipschitz gradients.

We motivate our focus on smooth $f$ with several applications; software reliability testing \cite{cai2004optimal}, stratified sampling for Monte Carlo integration \cite{carpentier2015adaptive}, and estimating the norm of a matrix vector product $\|Ax\|_2^2$.
In software reliability testing \cite{cai2004optimal}, we are interested in estimating the failure rate of a piece of software across $n$ classes or tasks.
This is formulated as estimating $f(\bx) = \sum_{i=1}^n p_i \mu_i$, where $\mu_i$ is the probability that a test from class $i$ fails, and $p_i$ is the fraction of time that a task in class $i$ will be used.
Since $\bx$ is unknown a priori, we adaptively select what subroutines to test in order to estimate $f(\bx)$.
We note that in this setting the difficulty of the problem stems not from the complexity of $f$, as it is linear, but rather from the unknown variance of our observations.
These variances amplify the estimation errors in different coordinates by differing amounts, and so the optimal sampling frequencies critically depend on these unknown variances

Another application of smooth function estimation lies in the basic linear algebraic task of estimating the power in a matrix vector product $\|Ax\|_2^2$.
For $A\in \R^{n\times d}$, we see that defining $\bx := Ax$ we have that $dA_{i,J}x_J$ is an unbiased estimator of $A_i^\top x$ when $J\sim\text{Unif}([d])$.
This primitive of converting a computational problem to one of statistical estimation was formalized in \cite{bagaria2021bandit}.
This allows us to convert the computational problem of estimating $\|Ax\|_2^2$ into the problem of estimating $f(\bx) = \|\bx\|_2^2$ for an unknown vector $\bx$ that can only be sampled coordinate-wise via a noisy oracle.

We operate in the stochastic multi-armed bandit setting, a common framework for studying sequential decision making problems \cite{lai1985asymptotically}.
In this setting the vector of arm means $\bx\in \R^n$ is fixed but unknown, and at every time step we are able to query one coordinate of $\mu_i$, an arm, and obtain a noisy but unbiased sample, an arm pull \cite{bubeck2012regret}.
A prototypical multi-armed bandit problem is best arm identification, which has been extensively studied \cite{jamieson2014best}, and can be posed as $f(\bx) = \argmax_{i \in [n]} \mu_i$ where $f\colon\R^n\mapsto [n]$ for $[n] := \{1,\hdots,n\}$.
With this discrete valued output, the error between $f(\bx)$ and $f(\bxh)$ is taken as 0 if they are equal and 1 otherwise, with the same 0-1 loss being used across other pure exploration multi-armed bandit variants \cite{lattimore2020bandit}.

For smooth $f$, however, we are better able to analyze the function error caused by our $\bx$ estimation error.
We obtain via a Taylor expansion that a plug-in estimator $\bxh$ will achieve an error of $f(\bxh)-f(\bx)= \nabla f(\bx)^\top (\bxh - \bx) + O(\|\bxh-\bx\|_2^2)$.
This suggests that in the high accuracy regime where $\|\bxh-\bx\|_2$ is small, it is beneficial to sample coordinates with larger partial derivatives more.
Such Taylor expansions do not arise in classical multi-armed bandit variants like the best arm identification problem, as examining its real valued analogue $f(\bx) = \max_{i\in[n]} \mu_i$, we see that $\nabla f(\bx) = e_{i^{*}}$, where $i^{*}=\argmax_{i\in[n]} \mu_i$ (assumed to be unique).
If $\bx$ is well estimated such that $i^{*}$ has been identified, then the only way to improve our estimation of $f(\bx)$ is to improve our estimation of $\mu_{i^{*}}$.
This agrees with intuition but is not very useful, as in order for this dependence to hold, we need to have identified $i^{*}$, as $f$ is not everywhere differentiable nor are its gradients smooth.
In the setting with smooth $f$ however, we can guarantee that our estimate of the gradient $\nabla f(\bx)$ will improve as our estimate of $\bx$ improves.

\subsection{Outline}
In \Cref{sec:relatedWorks} we discuss related work and place this paper within the broader multi-armed bandit literature.
We formalize our problem setting in \Cref{sec:formulation}.
In \Cref{sec:algdesign} we propose our algorithm for adaptive function approximation.
We extend our analysis to more general scenarios in \Cref{sec:algExt}.
In \Cref{sec:apps} we numerically validate the theoretical improvements afforded by our scheme.
We conclude in \Cref{sec:conc}.
All theoretical claims are proved in the Appendix.

\section{Related work} \label{sec:relatedWorks}

The study of multi-armed bandits has a long history dating back to the 1950's \cite{robbins1952some,paulson1964sequential,lai1985asymptotically}.
Numerous variants of this fundamental problem have been studied, including many identification-based objectives \cite{lattimore2020bandit}.
Regarding estimation of a known functional of the unknown mean vector, several previous works have studied the case of linear functions $f$ with heteroskedastic noise, which we highlight below.
While the aim of our work is on constructing an algorithm with a stopping condition and fixed confidence guarantees,
these alternative schemes instead focus on regret and anytime performance and do not provide such guarantees.

\subsection{Software reliability assessment}
One well studied problem that fits as a special case of our results is software reliability testing \cite{cai2004optimal}.
In this setting there are $n$ classes of tests $C_1,\hdots, C_n$ that can be run on a piece of software. 
Tests from class $i$ fail independently with an unknown but fixed probability $\mu_i$. 
The system is known to be run under a given operational profile $\{p_i\}_{i=1}^n$, where $p_i$ indicates the fraction of time that a task in class $C_i$ will be used.
In this setting the objective is to estimate the unreliability of the system, which is $f(\bx) = \b{p}^\top \bx$.
Works in the adaptive software reliability estimation literature utilize the plug-in estimator $f(\bxh)$ which is unbiased due to the linearity of $f$,
and so the primary objective is to minimize the estimator's variance.
This is accomplished by a careful finite sample analysis of the corresponding Markov Chain that is induced by all possible sampling patterns \cite{hu2013enhancing}.
Since each arm fails with probability $\mu_i$ and succeeds with probability $1-\mu_i$, the noise in this problem is heteroskedastic.
Asymptotic sampling frequencies are derived by using a Lagrangian for constrained optimization \cite{lv2014asymptotic}.
These works provide a specialized analysis for the setting with Bernoulli noise, and can only accommodate linear functions $f$.

\subsection{Stratified sampling bandit literature}

A similar vein of literature has arisen in works on adaptive stratified sampling for Monte-Carlo integration \cite{carpentier2012adaptive}.
In this setting the objective is to integrate a function $f$ over a domain $\CX$ by partitioning the domain into strata and adaptively allocating samples within.
Assuming that we can compute the measure of each strata $i=1,\hdots,n$, the objective can be formulated as estimating $f(\bx) = \sum_i w_i \mu_i$, where $w_i$ is the measure of strata $i$ and $\mu_i$ is the expectation of $f$ over strata $i$.
Due to the different and unknown variances $\sigma_i^2$ of $f$ over different strata, adaptivity is required to learn the optimal sampling frequencies $\alpha_i \propto w_i \sigma_i$, the well known Neyman allocation \cite{lohr2019sampling}, in order to obtain sublinear regret.
An important distinction between the Monte Carlo integration problem and our setting is that in these works the objective is to estimate $\E \{f(X)\}$.
In our setting however, we are interested in evaluating $f$ at some ground truth value $\bx = \E\{X\}$ of which we are only able to obtain noisy coordinate-wise observations.
This means that our objective is to estimate $f(\E\{X\})$, which only coincides with the previous objective when $f$ is linear.

As these stratified sampling works both focus on regret, they provide guarantees on the deviation of the empirical sampling ratios from the optimal ones.
However, it is unclear how to construct algorithmic stopping conditions for these regret-focused methods to allow them to operate in the fixed confidence regime.
Such stopping conditions are necessary for fixed confidence algorithms, which are significantly more useful in practical estimation settings.
In the software reliability estimation setting for example, we wish to run tests until we are 98\% sure that we've estimated the reliability to within 1\%; regret is not a notion of interest, as the primary goal is this fixed confidence guarantee.
Additionally these works can only handle linear $f$ with heteroskedastic noise. We are able to generalize to arbitrary differentiable $f$ with $L$-Lipschitz gradients, providing fixed confidence guarantees.

\subsection{Bandits as a computational tool}
Due to the ever increasing size of datasets, multi-armed bandit-based randomized algorithms have been recognized as a useful tool for constructing instance-optimal algorithms for computational tasks.
Dating back to works in Monte Carlo Tree Search \cite{kocsis2006bandit} and hyper-parameter tuning \cite{li2017hyperband}, adaptivity has been used to focus computational efforts towards promising solutions.
Formalized into the framework of Bandit-Based Monte Carlo Optimization \cite{bagaria2021bandit}, this reduction from computation to adaptive statistical estimation has been used to solve many problems including finding the medoid of a dataset \cite{bagaria2018medoids,baharav2019ultra,tiwari2020bandit}, $k$-nearest neighbor graph construction \cite{lejeune2019adaptive,mason2019learning,mason2021nearest}, Monte Carlo permutation-based multiple testing \cite{zhang2019adaptive}, and mode estimation \cite{singhal2020query}.
Most similarly, a rank-one estimation problem was solved in \cite{kamath2020adaptive} via an adaptive SVD, showing the utility of this technique for linear algebraic primitives, which we extend here to estimating $\|Ax\|_2^2$.

In order to realize statistical gains as wall-clock gains, it is advantageous to use algorithms that require fewer rounds of adaptivity and avoid extra synchronization steps.
Semi-adaptive algorithms have been studied in various contexts in the multi-armed bandit literature, first for a multiplayer setting \cite{hillel2013distributed}, then for halving based algorithms \cite{karnin2013almost}, and recently in more general settings \cite{karpov2020batched,karbasi2021parallelizing}, including for submodular optimization \cite{balkanski2018adaptive,esfandiari2021adaptivity}.
Such schemes exploit the efficiency of batched operations in modern computer architectures leading to dramatically faster algorithms \cite{baharav2019ultra,bagaria2021bandit}.
To this end, all algorithms proposed in this paper perform batched sampling, and show dramatic wall-clock improvements over fully sequential methods in our experiments.

\section{Problem Formulation} \label{sec:formulation}

We study the problem of minimizing the number of samples $\budget$ required to construct an estimator $\bxh$ of $\bx$ such that $|f(\bx)-f(\bxh)|\le \epsilon$ with probability at least $1-\delta$.
The function $f:\R^n\mapsto\R$ is known, and is differentiable with $L$-Lipschitz gradients, where we denote $\nabla_i f(\b{x}) := \frac{\partial}{\partial x_i} f(\b{x}):= g_i$.
$\bx$ is unknown, and at each time step $t\in \N$ we select $i\in [n]$ and obtain $\mu_i+Z_{i,t}$, where $\{Z_{i,t}\}$ are independent standard normals.
Our results are readily extendable to the case where $\{Z_{i,t}\}$ are 1-sub-Gaussian, but for clarity of exposition we present our results for the Gaussian case.
Given these three metrics of interest ($\ferr,\perr,\budget$), we focus on algorithms for the fixed confidence setting where we constrain both the error probability $\perr$ and the output error $\ferr$, and minimize the number of samples needed.

\section{Adaptive Sampling Scheme}\label{sec:algdesign}

To design our adaptive algorithm, we begin by examining the error incurred by a plug-in estimator $\bxh$ more closely.
Due to the $L$-Lipschitz gradients of $f$ \cite{bubeckBlog} this error satisfies
\vspace{-.1cm}
\begin{equation}
        \left|f(\hat{\bx}) - f(\bx) - \nabla f(\bx)^\top (\bxh - \bx)\right| \le \frac{L}{2}\|\bxh-\bx\|_2^2.\vspace{-.1cm}
\end{equation}
Since our estimator for $\mu_i$ is distributed as $\hat{\mu}_i \sim \CN(\mu_i,1/T_i)$ after $T_i$ samples, the second order error $\|\bxh-\bx\|_2^2$ is readily bounded by uniformly sampling the arms.
This indicates that we need to have a sufficiently accurate estimate of $\bx$ before our linearized first order approximation is accurate. 
Once the second order error has been sufficiently controlled, we examine the first order error $\nabla f(\bx)^\top (\bxh - \bx)$ by noting that
\begin{equation}
    \nabla f(\bx)^\top (\bxh-\bx) \sim \CN \left(0,\sum_{i=1}^n \frac{g_i^2}{T_i}\right). \label{eq:genFnormal}
\end{equation}
To guarantee that this error is less than $\epsilon$ in magnitude with probability at least $1-\delta$, Hoeffding's inequality indicates that we need to sample until $\sum_{i=1}^n g_i^2/T_i \le 2\epsilon^2/\log(2/\delta)$.
Minimizing our sample complexity in order to reach this condition, we define our sampling distribution $\ba \in \Delta^{n}$, the $n$ dimensional probability simplex, such that $T_i=\alpha_i T$.
Observe that for a fixed number of samples $T$ we can equivalently minimize over $\ba$ instead of $\{T_i\}$, yielding the optimal sampling distribution
\vspace{-.1cm}
\begin{equation*}
    \argmin_{\ba\in \Delta^n} \sum_i \frac{g_i^2}{\alpha_i T}
    \implies \alpha_i \propto |g_i|, \numberthis\vspace{-.1cm}
\end{equation*}
due to Sion's minimax theorem \cite{sion1958general} and first order conditions.
Sampling each coordinate proportional to the magnitude of its partial derivative $|g_i|$ effectively rescales our confidence interval on each arm inversely proportional to the scalar multiplying it.

If $\b{g}$ is known a priori, as is the case in the linear setting, then so is the optimal allocation, and hence adaptivity is not necessary.
In our setting however, the gradient $\b{g} = \nabla f(\bx)$ is unknown, as the point $\bx$ is unknown.
Motivated by the goal of sampling proportionally to the unknown gradient magnitudes, we propose \Cref{alg:generalf}.
It proceeds in rounds with doubling budgets $B_r = 2^r B_0$, and within each round uniformly samples the arms to yield improved estimation of $\bx$ and by extension $\nabla f(\bx)$, to estimate the optimal sampling frequencies.
At each iteration the algorithm checks the stopping condition in Line 10, determining whether sampling according to the estimated optimal sampling frequencies with the same budget would yield sufficiently small error.
If so, it performs these pulls and returns the plug-in estimate of $f(\bx)$.
Otherwise, it continues to the next round to refine its estimates of the optimal sampling frequencies.

Through our algorithm we construct estimators $\bxh^{(r)}$ of $\bx$ for estimating $\nabla f(\bx)$, where $\hat{\mu}_i^{(r)}$ is the average of the $\tilde{B}_r:=\sum_{\ell=1}^r\lceil B_\ell/n\rceil$ samples of arm $i$ taken prior to and including round $r$.
We utilize anytime confidence intervals of width $C_r$, where after $\tilde{B}_r$ pulls of arm $i$ we have that
\vspace{-.15cm}
\begin{equation}
    |\hat{\mu}_i^{(r)} - \mu_i| \le \sqrt{\frac{ 2\log (12nr^2/\delta)}{\tilde{B}_r}}:= C_r \vspace{-.1cm}
\end{equation}
with high probability, formalized in Lemma \ref{lem:conf}.
While tighter confidence intervals based on the law of the iterated logarithm can be utilized as in \cite{jamieson2014lil,kaufmann2016complexity}, in this work we use simple ones derived from Hoeffding's inequality for clarity of exposition. 

To estimate the importance of each coordinate, we construct upper and lower bounds for our gradient as
\vspace{-.125cm}
\begin{align} \label{eq:gradEst}
    \hat{g}_i^{(r,L)} &:= \min_{\b{y} : \|\b{y} - \bxh^{(r)}\|_\infty \le C_r} \left| \nabla_i f(\b{y})\right|,\nonumber\\
    \hat{g}_i^{(r,U)} &:= \max_{\b{y} : \|\b{y} - \bxh^{(r)}\|_\infty \le C_r} \left| \nabla_i f(\b{y})\right|. \vspace{-.125cm}
\end{align}
Assuming that our confidence intervals hold, we have that $\hat{g}_i^{(r,L)} \le g_i \le \hat{g}_i^{(r,U)}$ for all $i,r$.
Our estimators $\hat{g}_i^{(r,L)}$ presume that the partial derivative is as small as it could feasibly be.
This is critical to ensure that we do not oversample unimportant arms.
If this minimization is computationally intensive, one efficient alternative is to define $\hat{g}_i^{(r,L)}$ as $(|\nabla_i f(\bxh^{(r)})| - L\sqrt{n} C_r )_{+}$ where $(x)_{+} := \max(x,0)$.
This coarse bound ignores the local geometry of the problem and instead utilizes the global Lipschitzness bound $L$, but maintains the same guarantees for worst case functions.
We correspondingly can set $\hat{g}_i^{(r,U)}$ as $|\nabla_i f(\bxh^{(r)})| + L\sqrt{n} C_r$. These overestimates $\hat{g}_i^{(r,U)}$ ensure that we do not terminate too early.

\begin{algorithm}[h]
\begin{algorithmic}[1]
\caption{\label{alg:generalf} \texttt{Adaptive function approximation}}
\State \textbf{Input:} arms $[n]$, function $f$, target accuracy $\epsilon$, error probability $\delta$
\State Set $B_0 = 17Ln^2 \epsilon^{-1} \log(12n/\delta)$
\For{$r=1,2,\hdots$}
\State Set round budget $B_r = B_02^r$
\State Pull each arm $\lceil B_r/n\rceil$ times
\State Use all previous samples to construct $\bxh^{(r)}$
\State Compute bounds $\{\hat{g}_i^{(r,L)},\hat{g}_i^{(r,U)}\}$ as in \eqref{eq:gradEst}
\State Construct sampling frequencies $\alpha_i^{(r)} \propto \hat{g}_i^{(r,L)}$ 
\State Set $\tilde{T}_i^{(r)} :=\lceil (\alpha_i+\frac{1}{n})B_r\rceil$
\If{$\sum_i \frac{\left(\hat{g}_i^{(r,U)}\right)^2}{\tilde{T}_i^{(r)}} \le \frac{\epsilon^2}{16\log(6/\delta)} $} \label{line:alg1TermCond}
\State Pull arm $i$ $\tilde{T}_i^{(r)}$ times to construct $\bxt$
\State \Return{$f\left(\bxt\right)$}
\EndIf
\EndFor
\end{algorithmic}
\end{algorithm}

We provide the following theorem regarding the performance of Algorithm \ref{alg:generalf}. We express our results utilizing big O notation to highlight the dependence on relevant problem parameters. All constants are made explicit in \Cref{app:mainProof}.

\begin{thm} \label{thm:main}
Algorithm 1 succeeds with probability at least $1-\delta$ in outputting $\bxt$ such that 
$\left|f(\bxt) - f(\bx)\right| \le \epsilon$, using a number of arm pulls at most
\begin{equation*}
    O\left(\frac{\|\nabla f(\bx)\|_1^2 \log(1/\delta)}{\epsilon^2}+\frac{n^2L \log(n/\delta)}{\epsilon}\right).
\end{equation*}
\end{thm}

The proof of this theorem is relegated to \Cref{app:mainProof}, but we provide a sketch below. The proof begins by showing that the good event $\xi$, where the mean estimators $\{\hat{\mu}_i^{(r)}\}$ stay within their confidence intervals, occurs with high probability.
Then, we show that our error due to the second order terms is less than $\epsilon/2$ due to our choice of a sufficiently large $B_0$.
Next, we show that our sampling frequencies $\alpha_i^{(r)}$ are sufficiently large for important arms.
Critically, arms with small $g_i$ are sufficiently sampled from the uniform sampling in $\tilde{T}_i^{(r)}$, and arms with large $g_i$ will have large $\alpha_i^{(r)}$. 
Finally, we bound the number of rounds required until the termination condition is met.

In \Cref{app:l2tighter} we discuss how \Cref{alg:generalf} and its analysis can be improved for the special case where $f(\bx)=\|\bx\|_2^2$, where we are able to tighten our bounds on the lower order terms.
In this case we are able to improve our gradient bounds, defining $\hat{g}_i^{(r,L)} := (|\hat{\mu}_i^{(r)}| - 4C_r)_{+}$, where we no longer need to use the Lipschitzness of the entire gradient operator due to the separability of $f$.
More importantly, we are now able to subtract the mean of the second order error terms; concretely $\|\bxt^{(r)}\|_2^2 - \sum_{i=1}^n 1/\tilde{T}_i$ is now an unbiased estimator of $f(\bx)$.
This allows us to argue probabilistically that our second order errors average out, reducing our lower order sample complexity by a factor of $\sqrt{n}$.

\begin{cor} \label{cor:l2main}
The above variant of \Cref{alg:generalf} will with probability at least $1-\delta$ estimate $\|\bx\|_2^2$ to accuracy $\epsilon$ using a number of arm pulls at most
\begin{equation*}
    O\left(\frac{\|\nabla f(\bx)\|_1^2 \log(1/\delta)}{\epsilon^2} + \frac{n^{3/2}\log(n/\delta)}{\epsilon}\right).
\end{equation*}
\end{cor}

\subsection{Gain of adaptivity}
Examining the result of \Cref{thm:main}, we see that the first term stemming from the first order Taylor expansion will dominate when $H^2 > Ln^2 \epsilon$, where the problem complexity measure $H$ is defined as $H:= \|\nabla f(\bx)\|_1$.
This scaling with $n$ is to be expected, as $H$ is the sum of $n$ terms.
This naturally occurs in the high accuracy regime when $\epsilon\rightarrow 0$, where we show that a sample complexity of $\Theta(H^2 \log(1/\delta) \epsilon^{-2})$ is the optimal dependence on problem parameters, as formalized in our lower bound in \Cref{sec:lb} (\Cref{prop:lowerbound}).

Comparing against the uniform sampling baseline, we see that taking $T_i = T/n$ would require a budget of approximately $n\|\nabla f(\bx)\|_2^2 \log(2/\delta)\epsilon^{-2}$ pulls.
Hence, the ratio between the number of samples required by this uniform baseline and the number of samples required by our adaptive scheme scales as 
\begin{equation}
    \texttt{gain}(f;\bx) := \frac{n\|\nabla f(\bx)\|_2^2}{\|\nabla f(\bx)\|_1^2},
\end{equation}
which is bounded between 1 and $n$.
This lower limit is achieved if all the partial derivatives are equal, as then there is no imbalance in the gradient to exploit.
On the other extreme, if our gradient has only one nonzero entry we see that uniform sampling requires a factor of $n$ more samples than its adaptive counterpart.
Studying the special case of $\ell_p$ norms, we see that $\texttt{gain}(\|\cdot\|_1;\bx)=1$ as all coordinates matter equally, and $\texttt{gain}(\|\cdot\|_2^2;\bx)=n\|\bx\|_2^2/\|\bx\|_1^2\ge 1$.
In accordance with intuition, in \Cref{app:lp} we show that when the coordinates of $\bx$ are not all identical the gain of adaptivity in estimating $\|\bx\|_p^p$ is a monotonically increasing function of $p \in [1,\infty)$.

\section{Algorithmic Extensions}\label{sec:algExt}
We now discuss practical extensions to \Cref{alg:generalf}, showing how it can be generalized to accommodate broader scenarios of practical interest.

\subsection{Thresholding variant}

We first extend our algorithm to the setting where our objective is to minimize the number of samples required to determine whether our function value $f(\bx)$ is above a given threshold $\tau$ or not, i.e. to determine whether $f(\bx)> \tau$.
This is a useful primitive in settings where the desired task is not estimation but rather decision making, for example in certifying software reliability.
This thresholding objective can be thought of as providing an ``adaptive'' $\epsilon$, where fewer samples can be used if $f(\bx)$ is far from $\tau$.
Our proposed algorithm for this setting proceeds in a similar manner to \Cref{alg:generalf}, but at each round now explicitly computes an estimate $f(\bxt^{(r)})$ to check whether the estimated gap $|f(\bxt^{(r)}) - \tau|$ is large enough relative to a bound on the error $|f(\bxt^{(r)}) - f(\bx)|$.

\begin{algorithm}[h]
\begin{algorithmic}[1]
\caption{\label{alg:thresholding} \texttt{Adaptive thresholding}}
\State \textbf{Input: } arms $[n]$, function $f$, threshold $\tau$, error probability $\delta$
\For{$r=1,2,\hdots$}
\State Set round budget $B_r = n2^r$
\State Pull each arm $\lceil B_r/n\rceil$ times
\State Use all previous samples to construct $\bxh^{(r)}$
\State Compute bounds $\{\hat{g}_i^{(r,L)},\hat{g}_i^{(r,U)}\}$ as in \eqref{eq:gradEst}
\State Construct sampling frequencies $\alpha_i^{(r)} \propto \hat{g}_i^{(r,L)}$ 
\State Set $\tilde{T}_i^{(r)} :=\lceil (\alpha_i+\frac{1}{n})B_r\rceil$
\State Pull arm $i$ $\tilde{T}_i^{(r)}$ times to construct $\bxt^{(r)}$
\If{$C_r^{(f)} \le |f(\bxt^{(r)})-\tau| $}
\State Define estimate $\bxt:= \bxt^{(r)}$
\State \Return{$f\left(\bxt\right) > \tau$}
\EndIf
\EndFor
\end{algorithmic}
\end{algorithm}

For this adaptive thresholding algorithm we define and utilize confidence intervals $C^f_r$ for the function evaluation estimators $f(\bxt^{(r)})$ such that $|f(\bxt^{(r)}) - f(\bx)|\le C^f_r$ with high probability, where
\begin{align*}
C^f_{r,1} &:= \sqrt{2\log(24r^2/\delta)\sum_i \frac{\left(\hat{g}_i^{(r,U)}\right)^2}{\tilde{T}_i^{(r)}}}\\
C^f_{r,2}&:= \frac{ Ln \log (24nr^2/\delta)}{\tilde{B}_r}\\
C^f_r &:= C_{r,1} + C_{r,2}, \numberthis \label{eq:cThresh}
\end{align*}
where $C^f_{r,1}$ is our confidence interval on our first order error in round $r$, and $C^f_{r,2}$ is our confidence interval on our higher order error in round $r$.
For simplicity we assume that $f(\bx)\neq \tau$, but \Cref{alg:thresholding} can be modified to output $\perp$ in borderline cases where $f(\bx)$ is very close to $\tau$.
Concretely, the algorithm can output $\perp$ if it has not yet terminated and $C^f_r \le \epsilon/2$ for some input tolerance $\epsilon$ to indicate that $|f(\bx)-\tau| \le \epsilon$ with high probability.
This is similar to the indifference zone formulation in the operations research community \cite{kim2001fully}.
We provide the following theorem regarding the performance of \Cref{alg:thresholding}, where $\tilde{O}$ suppresses $\log \log$ terms in poly$\left(n,\delta,\|\nabla f(\bx)\|_1, (f(\bx)-\tau)^{-1} \right)$.

\begin{thm} \label{thm:thresholding}
\Cref{alg:thresholding} succeeds with probability at least $1-\delta$ and outputs whether $f(\bx)>\tau$, using 
\begin{equation*}
    \tilde{O}\left( \frac{\|\nabla f(\bx)\|_1^2\log(1/\delta)}{(f(\bx)-\tau)^2} + \frac{n^2L\log(n/\delta)}{|f(\bx)-\tau|}\right)
\end{equation*}
arm pulls, assuming that $f(\bx)\neq \tau$.
\end{thm}

The proof of this theorem, relegated to \Cref{app:thresh}, proceeds similarly to that of \Cref{thm:main}, with some interesting nuances stemming from the fact that our $\epsilon$ is now unknown.
Since we do not a priori know what accuracy to estimate $f(\bx)$ to, we need to sample according to our estimated optimal frequencies $\alpha_i^{(r)}$ in every round to generate $f(\bxt^{(r)})$ to estimate the necessary accuracy as $|f(\bxt^{(r)}) - \tau|$. 
Proceeding in rounds of doubling budgets, we analyze our sample complexity by bounding the number of rounds until our sampling frequencies are sufficiently accurate, our estimate of $\epsilon$ is sufficiently tight, and the number of samples is sufficiently large to yield the desired error.

\subsection{Adapting to unknown variances}

As stated, Algorithms 1 and 2 assume that the additive noise associated with the arm pulls is i.i.d. $\CN(0,1)$.
In many scenarios of practical interest however, noise variance may vary dramatically across arms.
In software reliability assessment for example, the arm reward distributions are Bernoulli with mean $\mu_i$, where the variances are a function of the unknown arm rewards, with $\sigma_i^2=\mu_i(1-\mu_i)$.
More generally, these unknown variances can be thought of as nuisance parameters, which we need to estimate in order to efficiently run our algorithm but do not need to output.
To adapt to this scenario we generalize our algorithm to accommodate bounded distributions by utilizing an empirical Bernstein-type bound \cite{maurer2009empirical}, subsuming the software reliability assessment setting.
Our doubly adaptive algorithm adapts to both the unknown arm variances $\sigma_i$ as well as the unknown gradient $\nabla f(\bx)$.
For simplicity of exposition, in this section we assume that our function is separable, i.e. $f(\bx) = \sum_{i=1}^n f_i(\mu_i)$, but all ideas extend to the general case as discussed at the end of this section.

In order to accommodate these unknown variances, our new algorithm utilizes variance adaptive confidence intervals of width $C_\mu(i,r)$ for our sample mean estimators $\hat{\mu}_{i}^{(r)}$ of $\mu_i$ in round $r$, defined in \eqref{eq:cmu}.
To quantify the improved rate of decay of our confidence intervals, and to sample coordinates with high variance more, we utilize confidence intervals for our sample standard deviation estimators $\hat{\sigma}_i^{(r)}$.
We construct these confidence intervals $C_\sigma(i,r)$ for our estimator $\hat{\sigma}_{i}^{(r)}$ of $\sigma_i$ in round $r$ as below:
\begin{align}
    \hspace{-.25cm}C_\sigma(i,r) &:= \sqrt{\frac{2\log (8nr^2/\delta)}{T_i^{(r)}-1}} \label{eq:csigma}\\
    \hspace{-.25cm}C_\mu(i,r) &:= \hat{\sigma}_i^{(r)}\sqrt{\frac{2 \log (32nr^2/\delta)}{T_i^{(r)}}} + \frac{7 \log (32nr^2/\delta)}{3(T_i^{(r)}-1)}.\hspace{-.05cm} \label{eq:cmu}
\end{align}
In accordance with our variance adaptive confidence intervals, we redefine our gradient estimators as 
\begin{align} \label{eq:gradEstVar}
    \hat{g}_i^{(r,L)} &:= \min_{|y_i - \hat{\mu}_i^{(r)}| \le C_\mu(i,r) } \left| f_i'(y_i)\right|, \nonumber\\
    \hat{g}_i^{(r,U)} &:= \max_{|y_i - \hat{\mu}_i^{(r)}| \le C_\mu(i,r) } \left| f_i'(y_i)\right|.
\end{align}
We note that we have similar bounds for our sample standard deviation estimates, as $\hat{\sigma}_i^{(r,L)} := \big(\hat{\sigma}_i^{(r)} - C_\sigma(i,r) \big)_{+} $ and $\hat{\sigma}_i^{(r,U)} = \hat{\sigma}_i^{(r)} + C_\sigma(i,r)$.
With these in hand we are able to define \Cref{alg:adaConf}, which operates in a manner similar to \Cref{alg:generalf}.
It proceeds in rounds of increasing accuracy, first coarsely estimating all the means and standard deviations via uniform sampling then refining estimates on higher variance arms by sampling proportionally to a (pessimistic) estimate of their standard deviation.
At the end of each round termination sampling allocations $\tilde{T}_i^{(r)}$ are constructed, and the algorithm samples according to these allocations and terminates if the following termination condition is met: \vspace{-.15cm}
\begin{align*}
    \hspace{-.1cm}\sum_{i=1}^n  \frac{\left(\hat{g}_i^{(r,U)} \hat{\sigma}_i^{(r,U)}\right)^2}{\tilde{T}_i^{(r)}} 
    &+ \frac{\epsilon}{3}\max_i \frac{\hat{g}_i^{(r,U)}}{\tilde{T}_i^{(r)}} \le \frac{\epsilon^2}{8\log(8/\delta)}. \hspace{-.1cm}\numberthis\label{eq:termCondAlg}
\end{align*}

\begin{algorithm}[h]
\begin{algorithmic}[1]
\caption{\label{alg:adaConf} \texttt{Adaptive function approximation under heteroskedastic noise}}
\State \textbf{Input: } arms $[n]$, function $f$, target accuracy $\epsilon$, error probability $\delta$
\State Set $B_0 = \frac{2n^2L\log(8n/\delta)}{\epsilon}$
\For{$r=1,2,\hdots$}
\State Set round budget $B_r = B_02^r$
\State Pull each arm $\lceil B_r/n\rceil$ times
\State Compute $\bxh^{(r)},\bsh^{(r)}$
\State Construct $\{\hat{g}_i^{(r,L)},\hat{g}_i^{(r,U)}\}$ as in \eqref{eq:gradEstVar}
\State Compute sampling frequencies $\beta_i^{(r)} \propto \hat{\sigma}_i^{(r,L)}$
\State Pull arm $i$  $T_i^{(r)} = \lceil (\beta_i^{(r)} + 1/n)B_r\rceil$ times
\State Update $\bxh^{(r)},\bsh^{(r)}$, recompute $\{\hat{g}_i^{(r,L)},\hat{g}_i^{(r,U)}\}$
\State Compute frequencies $\alpha_i^{(r)}\propto \hat{g}_i^{(r,L)} \hat{\sigma}_i^{(r,L)}$ 
\State Set $\tilde{T}_i^{(r)} :=\lceil (\alpha_i+\frac{1}{n}) B_r\rceil$
\If{ Termination condition in \eqref{eq:termCondAlg} is met}
\State Pull arm $i$ $\tilde{T}_i^{(r)}$ times to construct $\bxt$
\State \Return{$f\left(\bxt\right)$}
\EndIf
\EndFor
\end{algorithmic}
\end{algorithm}
\vspace{-.35cm}

We are able to state the following Theorem regarding the performance of  \Cref{alg:adaConf}, expressing our sample complexity in big Oh notation with respect to $\epsilon$ for clarity, but detailing all the lower order dependencies on $\{g_i\},\{\sigma_i\},n,\epsilon,$ and $\delta$ with explicit constants in \Cref{app:adavar} (\Cref{lem:adaConfBudget}).
\begin{thm} \label{thm:algbounded}
\Cref{alg:adaConf} successfully outputs $f(\bxt)$ such that $|f(\bxt)-f(\bx)|\le \epsilon$ with probability at least $1-\delta$ using a number of arm pulls at most
\begin{equation*}
    O\left( \frac{\left(\sum_i g_i \sigma_i\right)^2 \log(1/\delta)}{\epsilon^2}\right) + o\left(\epsilon^{-2}\right).
\end{equation*}
\end{thm}

We prove \Cref{thm:algbounded} in \Cref{app:adavar}, beginning by showing that the estimators $\hat{\mu}_i^{(r)}$ and $\hat{\sigma}_i^{(r)}$ stay within their confidence intervals.
This yields sampling frequencies $\beta_i^{(r)}$ that sample high variance arms sufficiently, which in turn results in $\alpha_i^{(r)}$ that ensure that arms with large $g_i$ and $\sigma_i$ are sampled enough.
We then show that once sufficiently many rounds have elapsed the algorithm will terminate and output an $\epsilon$ accurate estimate.
Additionally, while the sampling frequencies $\alpha_i^{(r)}$ in \Cref{alg:adaConf} are constructed with simplicity and asymptotic optimality in mind,
we discuss further optimizations in \Cref{app:adavar} to improve lower order terms in $\epsilon$. 

For the sake of simplicity, in this section we assumed that our function $f$ is separable.
Our algorithm can be modified to accommodate general $f$ by adding into our $\beta_i^{(r)}$ sampling frequencies a term proportional to $(\hat{\sigma}_i^{(r,L)})^2$ to minimize the maximum confidence interval width, as in \cite{antos2010active}.
While here we handled the case of bounded noise distributions, our results can additionally be extended to the case of additive gaussian noise with unknown noise variances $\sigma_i^2$.

\subsection{Lower bound} \label{sec:lb} 

As previously discussed, our algorithms appears to be order optimal in terms of all relevant problem parameters in the high accuracy regime.
To this end we provide a matching lower bound for the linear setting where $f(\bx) = \b{g}^\top \bx$, a plug-in estimator is used, and our arm pulls for arm $i$ are corrupted by additive Gaussian noise with variance $\sigma_i^2$.

\begin{prop}\label{prop:lowerbound}
In order to estimate $f(\bx) = \b{g}^\top \bx$ to accuracy $\epsilon$ with error probability at most $\delta$, where pulls from arm $i$ are corrupted by independent additive Gaussian noise with variance $\sigma_i^2$, a plug-in estimator based on any static sampling allocation $\{T_i\}$ where $\sum_{i=1}^n T_i = T$ requires
\begin{equation*}
    T\ge \left(\sum_i |g_i| \sigma_i\right)^2 \epsilon^{-2} \log(1/4\delta).
\end{equation*}
\end{prop}
We conjecture that this lower bound extends beyond linear functions to general $f$ with $L$-Lipschitz gradients.
This bound should be expected to hold, as in the high accuracy regime where $\epsilon\to0$ the function is essentially linear within the region of estimation.
Similar techniques can be used to extend this lower bound to static allocations for such $f$ when $\epsilon$ is sufficiently small and $T_i=\omega(\epsilon^{-1})$, as then the error in the linear approximation can be bounded with probability at least $1-\delta$ as $\frac{L}{2}\|\bxh-\bx\|_2^2 \ge \frac{L}{2} \max_i T_i^{-1} \log(1/\delta) = o(\epsilon)$.
The technical difficulty in providing an information theoretic lower bound in this more general setting is that the number of pulls per arm cannot be independently bounded, as the estimation error accumulates jointly across the coordinates.
This lower bound, proved in \Cref{app:lb}, shows that our scheme's dependence on all relevant problem parameters $\{g_i\}, \{\sigma_i\},\epsilon,$ and $\delta$, is order optimal in the linear case.

\section{Numerical Experiments} \label{sec:apps}

We corroborate our theoretical claims with numerical simulations to show the practical utility and desirable finite sample performance of our scheme.
Note that, while not shown in these plots, one significant advantage of our schemes over optimisim-based ones is that our schemes offer valuable fixed-confidence stopping conditions.
Additional experimental details can be found in \Cref{app:expdetails}.

\subsection{Estimating $\|Ax\|_2^2$} 
Here we tackle the problem of estimating the power of a matrix vector product, which arises in statistical testing and Fourier signal processing applications.
In our simulations we compare our \Cref{alg:generalf} with a uniform sampling benchmark, which pulls all arms $T/n$ times for a total budget of $T$.
Due to the nonlinearity of $f(\bx) = \|\bx\|_2^2$, we cannot compare with the software reliability estimation algorithms or the optimism-based regret minimization ones.
We are able to construct a modified optimism-based algorithm for this setting, which we discuss and show our improvement over in \Cref{app:extraexp}.

\begin{figure}[h]
 \vspace{-.2cm}
    \centering
    \begin{subfigure}[b]{0.47\linewidth}
    \includegraphics[width=\linewidth]{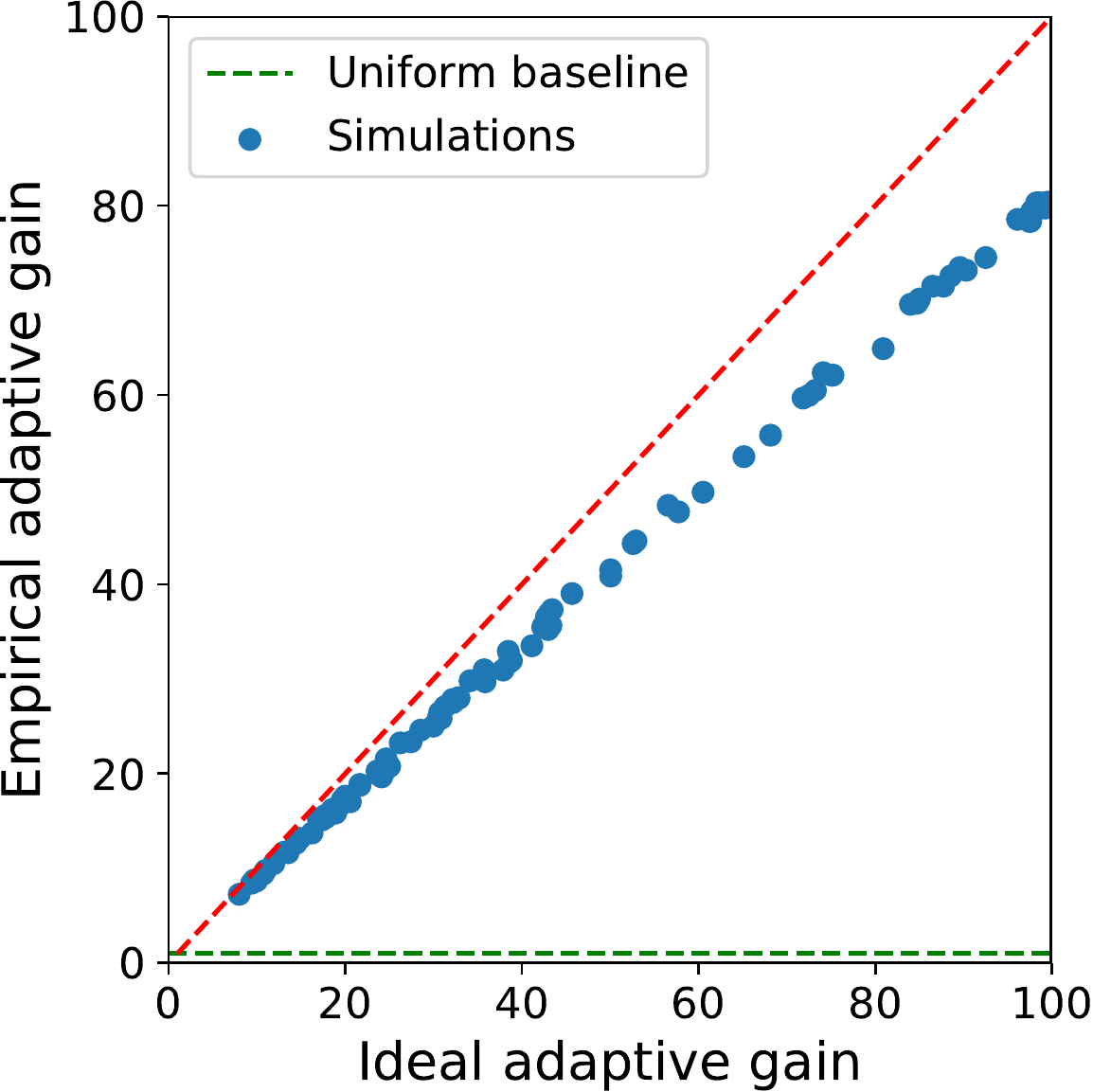}
    \vspace{-.5cm}
    \subcaption{}
    \label{fig:gainPlot_l2}
    \end{subfigure}
    \hfill
    \begin{subfigure}[b]{0.47\linewidth}
    \includegraphics[width=\linewidth]{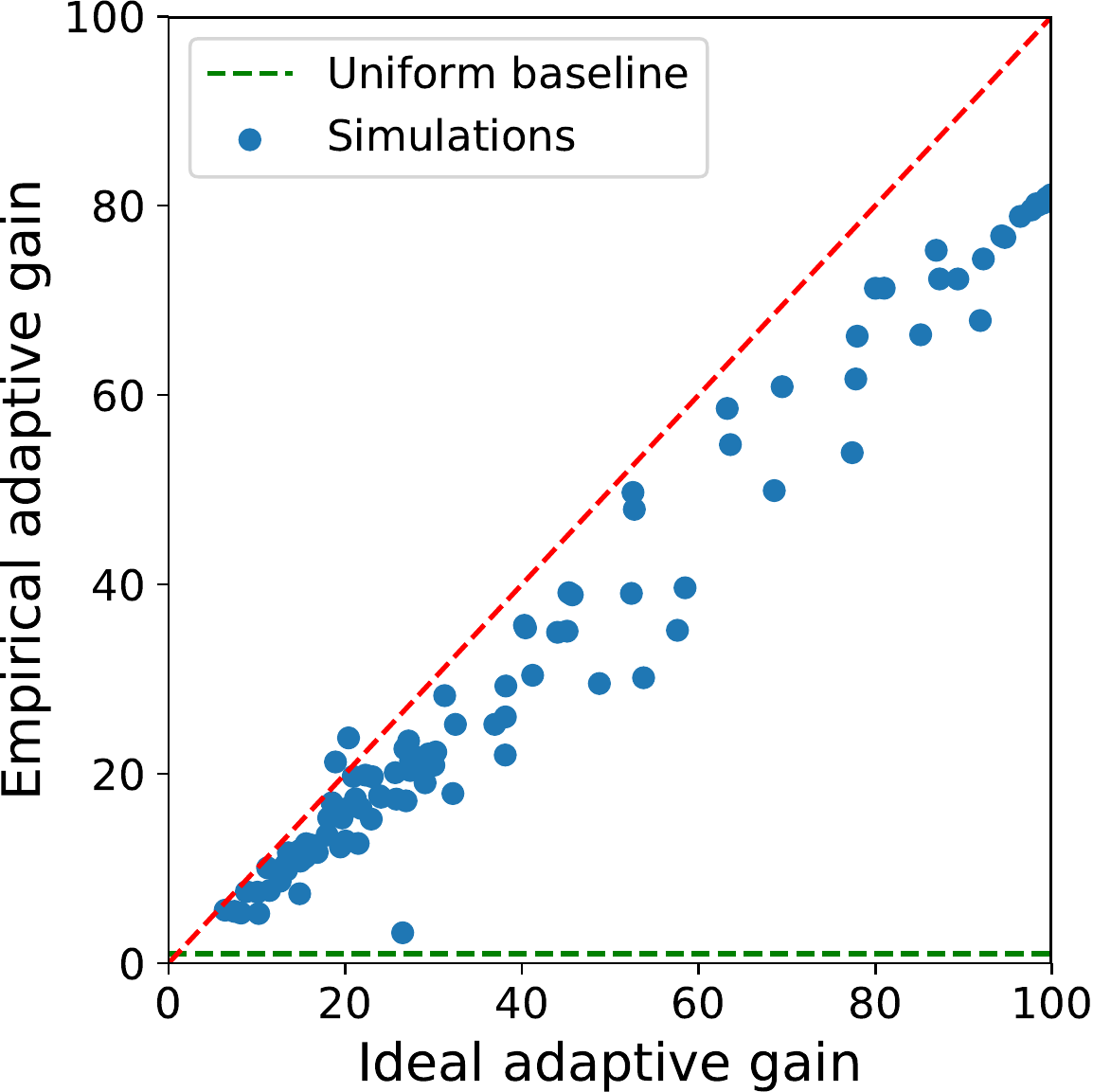}
    \vspace{-.5cm}
    \subcaption{}
    \label{fig:gainPlot_bern}
    \end{subfigure}
    \vspace{-.3cm}
    \caption{
    Time until termination for (a) $f(\bx)=\|\bx\|_2^2$ and (b) $f(\bx) = \b{g}^\top \bx$ for $\epsilon= 10^{-4},\delta= .01$.
    Dashed red line indicates theoretically predicted gain from adaptivity, the line $y=x$.
    100 randomly generated problem instances are shown.}
    \label{fig:gainPlot}
    \vspace{-.3cm}
\end{figure}
In \Cref{fig:gainPlot} we plot the improvement in stopping time afforded by adaptivity.
Taking the ratio of the number of samples required by the uniform sampling based algorithm to terminate to the number of samples required by our adaptive sampling based algorithm to terminate, we see that these ratios follow a linear trend.
In these plots we took $\tilde{T}_i^{(r)} = \lceil 10\alpha_i B_r\rceil$, which does not orderwise impact our theoretical guarantees but allows for superior finite sample performance.
While theory would predict that all points fall on the line $y=x$, finite sample issues and our choice of $\tilde{T}_i^{(r)}$ yield the resulting slope of $\approx 10/12$, for which we provide further discussion in \Cref{app:expdetails}.
Additionally, we reduce the confidence intervals constants to allow for more aggressive sampling.
All modifications from the written algorithm are detailed in \Cref{app:changesAlg}.

\begin{figure}[h]
    \centering
    \includegraphics[width=\columnwidth]{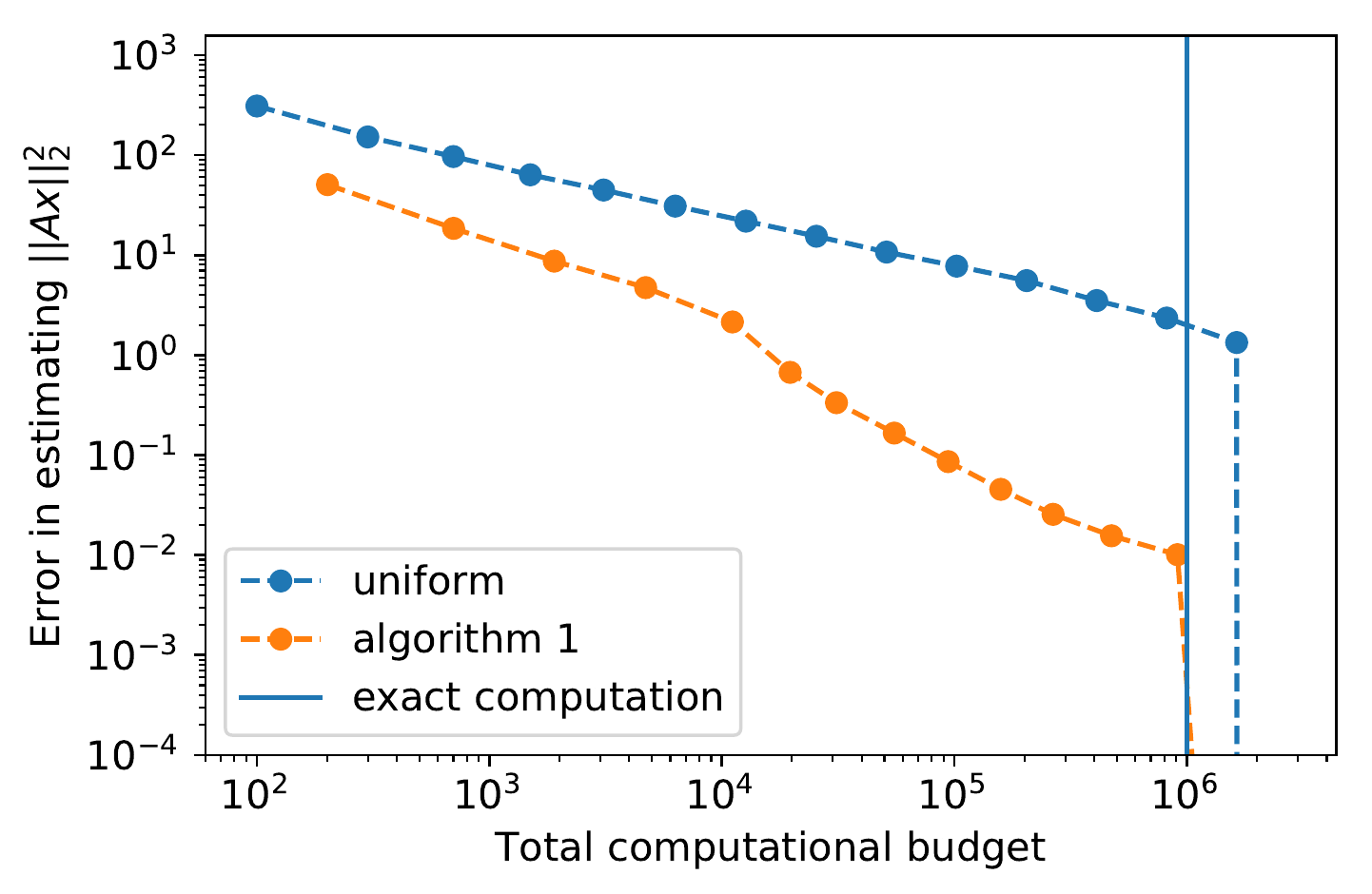}
    \vspace{-.8cm}
    \caption{Estimation error of $\|Ax\|_2^2$ for matrix $A$, vector $x$, for $n=100$ and $d=10k$.}
    \label{fig:Ax_n100_d10k_pareto}
    \vspace{-.2cm}
\end{figure}

In \Cref{fig:Ax_n100_d10k_pareto} we see the dramatic gain afforded by adaptivity in the estimation of $\|Ax\|_2^2$.
For this random example the gain between uniform and adaptive sampling was projected to be $60.5$.
While in these experiments we utilize $A,x$ such that the variance of our drawn samples is approximately $1$ for simplicity, in many applications one would need to adapt to the unknown variance of different entries.
In order to accomplish this, a bound on $\max_{i,j} |A_{ij} x_j|$ is required.
This allows for adaptation to the unknown variances $\sigma_i^2 = d \sum_{j=1}^d (A_{ij}-\mu_i)^2$.

Note that in this computational setting, we have that in addition to high probability bounds on the sample complexity of the algorithm, we can provide bounds on the expected runtime as well, as the total number of samples required is upper bounded by $nd$ with probability 1.

\subsection{Linear $f$ with heteroskedastic noise} 

We provide further numerical validation of our method by demonstrating its performance in the software reliability estimation setting, where we wish to estimate $f(\bx) = \b{g}^\top \bx$, and when we query arm $i$ we obtain a Bernoulli sample with mean $\mu_i$ (heteroskedastic noise) indicating whether the test failed.
We see that our general algorithm is able to outperform the optimism-based one of \cite{carpentier2015adaptive} in both error and wall-clock time, while additionally providing stopping time guarantees.

\begin{figure}[h]
    \centering
    \vspace{-.3cm}
    \includegraphics[width=\columnwidth]{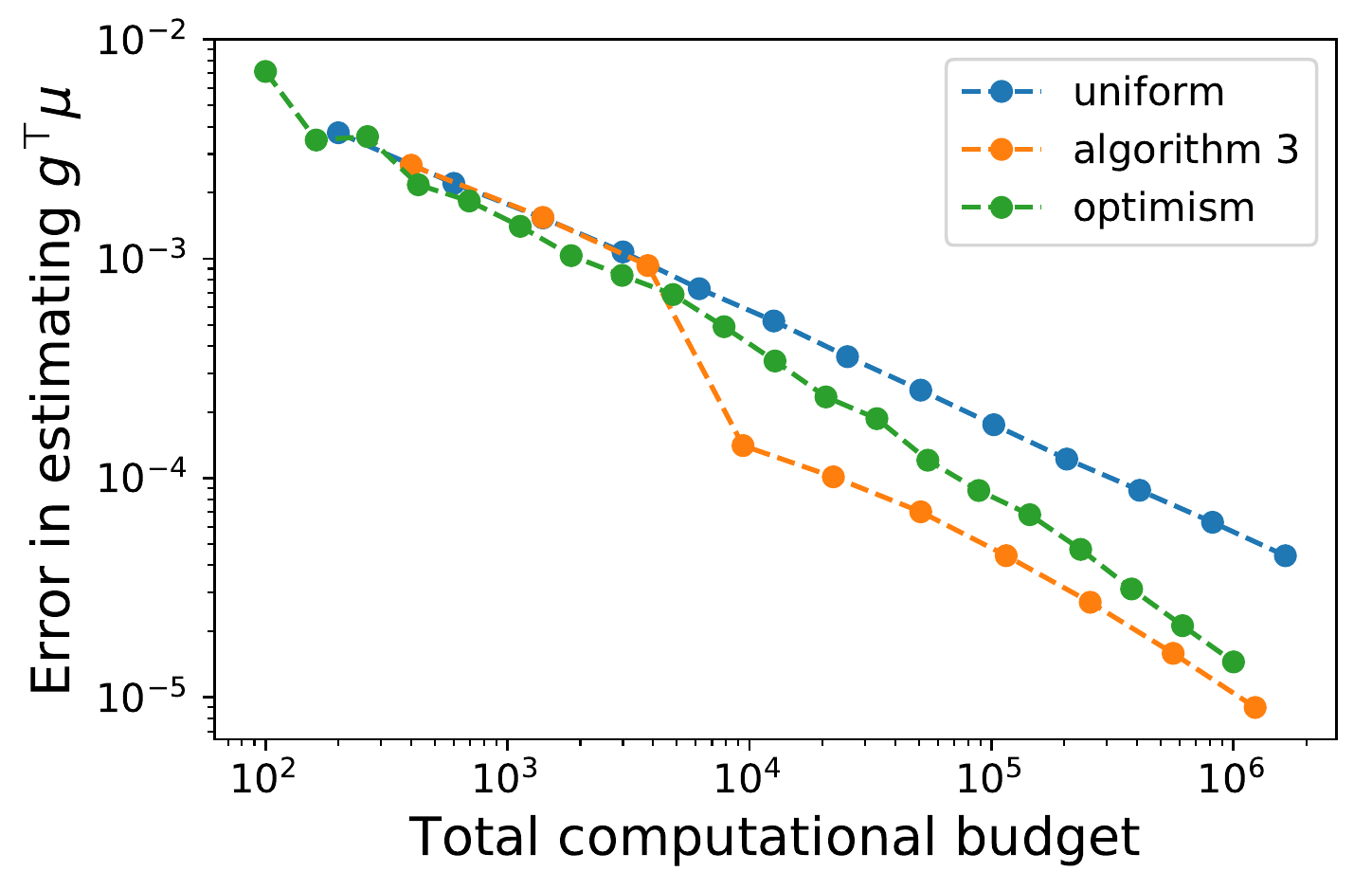}
    \vspace{-.7cm}
    \caption{Estimation error of $\b{g}^\top \bx$ under Bernoulli noise for $n=100$, $g_i$ uniform in [0,1] and $\bx$ normalized pareto distribution.}
    \label{fig:bernoulli_n100_pareto}
    \vspace{-.3cm}
\end{figure}

In \Cref{fig:bernoulli_n100_pareto} we see the favorable comparison of \Cref{alg:adaConf} to the optimism based method.
Even though \Cref{alg:adaConf} is designed with a fixed confidence stopping condition in mind, and does not optimize for this fixed budget metric of error after $T$ samples, we see that it is in fact able to slightly outperform the optimism-based algorithm.
Both adaptive algorithms exhibit an inflection point (which is much sharper in \Cref{alg:adaConf}) when they identify the most important coordinate with largest $g_i \mu_i$ and are able to use the majority of their samples on it.
This effect can also be seen to some degree in \Cref{fig:Ax_n100_d10k_pareto}.
Additional numerical experiments can be found in \Cref{app:extraexp}.

\begin{table}[h] 
\begin{tabular}{l|l|l|l}
                                                           & Uniform         & Algorithm 3                                         & Optimism \\ \hline
$T=10^4$   & $3.4\times 10^{-3}$ & $7.0\times 10^{-3}$ & $3.2\times 10^{-1}$     \\ \hline
$T=10^5$  & $5.2\times 10^{-3}$                & $8.4\times 10^{-3}$                                                  & $3.1$       \\ \hline
$T=10^6$ & $7.5\times 10^{-3}$                 & $1.1\times 10^{-2}$                                                   & $3.1 \times 10$     
\end{tabular}\vspace{-.2cm}
\caption{Wall-clock time comparison for estimation of $f(\bx) = \b{g}^\top \bx$ as in \Cref{fig:bernoulli_n100_pareto}, averaged over 100 trials. Simulation details in \Cref{app:expdetails}.}\label{table:times}
\vspace{-.2cm}
\end{table}

We see that while the statistical performance of the optimism-based algorithm is comparable to our scheme, it is computationally much slower.
This is shown in \Cref{table:times}, where we see that our algorithm is over 3 orders of magnitude faster when run for $T=10^6$ samples, and is almost as fast as the batched pure uniform sampling algorithm.
The runtimes of the two batched methods increase extremely slowly (logarithmically) as the budget increases, while the fully sequential optimism based algorithm suffers from a runtime linear in $T$.

\section{Conclusions and Future Work} \label{sec:conc}
In this paper, we considered the problem of estimating the value of a known smooth function $f$ at an unknown point $\bx\in\R^n$.
We provided a novel algorithm to solve this problem, showing that adaptively estimating the important coordinates of the unknown point can yield dramatic performance improvements.
We provided a matching lower bound in the linear setting, showing that our algorithm achieves order optimal performance in terms of all relevant problem parameters.
We demonstrated the broad applicability of our scheme by showing its extensions to a thresholding decision setting, as well as our algorithms' ability to adapt to unknown variances.

There are several interesting directions for future work.
We showed in the special case where $f(\bx)=\|\bx\|_2^2$ that additional structure can lead to a more efficient algorithm.
One natural question is what other classes of functions can yield improved rates.
Additionally, while in this work we only considered bounded noise distributions (or Gaussian noise with unknown variance), one should be able to use more sophisticated estimators like Catoni's M-estimator \cite{catoni2012challenging} to accommodate more general noise distributions with moments of order $1+\epsilon$ for $\epsilon\in(0,1]$ \cite{bubeck2013bandits}.
Finally, the technique we use for estimating $\|Ax\|_2$ can also be utilized for the estimation of $\|AB\|_F^2$, as this can be viewed as the $\ell_2$ norm of the vectorized resulting matrix.
This matrix matrix product contains additional structure, and has been studied from alternative perspectives in the Monte Carlo matrix multiplication literature \cite{drineas2006fast}, so designing wall-clock efficient algorithms for this setting is an interesting question for future work.

\bibliographystyle{apalike}
\balance 
\bibliography{mybib}

\begin{appendix}
\onecolumn

\section{$\ell_p$ norms}\label{app:lp}

In this Appendix we analyze our algorithms' performance for $\ell_p$ norms, where we wish to estimate $f(\bx) = f(\bx;p) := \|\bx\|_p^p$ to accuracy $\epsilon$.
We focus on the high accuracy regime for the sake of exposition, taking $\epsilon\rightarrow 0 $ and ignoring $o(\epsilon^{-2})$ terms.
For $\|\bx\|_\infty$, in the high accuracy regime it is order optimal to first identify an $\epsilon/2$ optimal arm $i^{*}$ such that $|\mu_{i^{*}}| \ge \max_i |\mu_i| - \epsilon/2$, and then estimate $\mu_{i^{*}}$ to accuracy $\epsilon/2$.
This two step procedure has an identification cost of $\tilde{O}(\sum_{i=1}^n (\Delta_i^{(\epsilon)})^{-2})$ for $\Delta_i^{(\epsilon)}= \max(\Delta_i,\epsilon)$ \cite{jamieson2014best}, after which the identified arm can be pulled $\tilde{O}(\epsilon^{-2})$ times to estimate its mean to accuracy $\epsilon/2$, yielding an overall sample complexity of $\tilde{O}(\epsilon^{-2} + \sum_{i=1}^n (\Delta_i^{(\epsilon)})^{-2}) \to \tilde{O}(\epsilon^{-2})$, assuming that only a constant number of arms attain this maximum.
Here, $\tilde{O}$ hides log factors in poly$(n,1/\delta,1/\epsilon,1/\Delta_i)$.
This is in contrast to a uniform sampling based algorithm, which would estimate the means of all arms to accuracy $\epsilon/2$ and then output the largest value, requiring $\tilde{O}(n\epsilon^{-2})$ samples.
This yields a performance improvement afforded by adaptivity of 
$\tilde{O}(n)$ when the best arm is unique.
As shown in \Cref{thm:main}, our adaptive algorithm requires $O(\|\nabla f(\bx)\|_1^2\epsilon^{-2}\log(1/\delta))$ samples for smooth $f$ ignoring lower order terms in $\epsilon$, as opposed to a uniform sampling algorithm which will require $O(n \epsilon^{-2}\|\nabla f(\bx)\|_2^2\log(1/\delta))$, yielding a ratio of $O(n\|\nabla f(\bx)\|_2^2 / \|\nabla f(\bx)\|_1^2)$.
On the other end of the spectrum, the case of $f(\bx)=\|\bx\|_1$ is surprising in that adaptivity does not save samples, as for all $i$ such that $\mu_i\neq0$ we have that $|\nabla_i f(\bx)| = 1$.
Thus, all coordinates are equally important, and it is optimal to uniformly sample the entries.
Revisiting  $f(\bx)=\|\bx\|_2^2$, we have that $\nabla f(\bx) = 2 \bx$, and so the gain from adaptivity will be $O(n\|\bx\|_2^2 / \|\bx\|_1^2)$, potentially a factor of $n$ improvement.
Defining the gain afforded by adaptivity for a fixed vector $\bx$ as
\begin{equation} \label{eq:hp}
     \texttt{gain}(\|\cdot\|_p^p,\bx) := \frac{n\|\nabla f(\bx;p)\|_2^2}{\|\nabla f(\bx;p)\|_1^2},
\end{equation}
the asymptotic gain of adaptive sampling over uniform sampling in the limit as $\epsilon \to 0$, we are able to state the following proposition showing that the gain increases with $p$ as the function becomes less smooth.
\begin{prop}\label{prop:lp}
For $p \in [1,\infty)$, the sample complexity improvement  $\texttt{gain}(\|\cdot\|_p^p;\bx)$ in \eqref{eq:hp} is a monotone nondecreasing function of $p$, and a monotone increasing function if the $\mu_i$ are not all identical.
\end{prop}

\begin{proof}[Proof of \Cref{prop:lp}]
Analyzing our gain, we see that we can assume without loss of generality that $\mu_i\ge0$ for all $i$.
Defining $h(\bx;p) := \sum_{i=1}^n \mu_i^{2(p-1)}$ and $g(\bx;p) := (\sum_{i=1}^n \mu_i^{p-1})^2$, we have that
\begin{equation}
    \texttt{gain}(\|\cdot\|_p^p,\bx) = \frac{n\|\nabla f(\bx;p)\|_2^2}{\|\nabla f(\bx;p)\|_1^2} = \frac{n \sum_{i=1}^n \mu_i^{2(p-1)}}{\left(\sum_{i=1}^n \mu_i^{p-1}\right)^2} = \frac{n h(\bx;p)}{g(\bx;p)}.
\end{equation}
To show that this is a nondecreasing function of $p$, we simply need to show that $h'(\bx;p)g(\bx;p) -h(\bx;p) g'(\bx;p)\ge 0$, where the derivatives are taken with respect to $p$ for a fixed $\bx$.
Since coordinates where $\mu_i=0$ affect neither the numerator nor the denominator we can remove them from consideration, and assume without loss of generality that $\mu_i>0$ for all $i$.
Taking the derivatives and simplifying, we have that our gain is a non decreasing function of $p$ as
\begin{align*}
\left(\sum \mu_i^{2(p-1)} \ln \mu_i\right)& \left(\sum \mu_i^{p-1}\right) - \left(\sum \mu_i^{2(p-1)}\right) \left(\sum \mu_i^{p-1} \ln \mu_i\right)\\
&=\left(\sum_{i,j} \mu_i^{2(p-1)} \mu_j^{p-1} \ln \mu_i \right) - \left(\sum_{i,j} \mu_i^{2(p-1)} \mu_j^{p-1} \ln \mu_j\right) \\
&\underset{(a)}{=} \sum_{i<j}\left(\left( \mu_i^{2(p-1)} \mu_j^{p-1} \ln \mu_i + \mu_j^{2(p-1)} \mu_i^{p-1} \ln \mu_j\right) - \left(\mu_i^{2(p-1)} \mu_j^{p-1} \ln \mu_j + \mu_j^{2(p-1)} \mu_i^{p-1} \ln \mu_i \right)\right)\\
&= \sum_{i<j}\left(\mu_i^{2(p-1)} \mu_j^{p-1} \ln \left(\mu_i/\mu_j\right) + \mu_j^{2(p-1)} \mu_i^{p-1} \ln \left(\mu_j/\mu_i\right)\right)\\
&= \sum_{i<j}\left(\mu_i^{p-1} \mu_j^{p-1} \ln \left(\mu_i/\mu_j\right) \left(\mu_i^{p-1} - \mu_j^{p-1} \right)\right)\\
&\ge 0. \numberthis
\end{align*}
Where in (a) we grouped pairs of terms $(i,j)$ and $(j,i)$ from both sums, and in the last inequality we observe that for $\mu_i,\mu_j>0$, the expression is always nonnegative.
This shows that the derivative of the gain with respect to $p$ is nonnegative for all $p\in[1,\infty)$.
Examining this final line more closely, we see that unless $\mu_i = \mu_j$ for all pairs $(i,j)$ there will be a positive term in this sum, indicating that the gain is a monotone increasing function of $p$ when the $\mu_i$ are not all identical.
\end{proof}

\section{Proofs of \Cref{alg:generalf}} \label{app:mainProof}

\textbf{Proof sketch:} we define the good event $\xi$ where the mean estimators $\{\hat{\mu}_i^{(r)}\}$ stay within their confidence intervals, and show in \Cref{lem:conf} that this event occurs with probability at least $1-\delta/3$.
This yields control of our gradient bounds $\{\hat{g}_i^{(r,L)},\hat{g}_i^{(r,U)}\}$, and guarantees that the sampling frequencies $\alpha_i^{(r)}$ are sufficiently large for important arms.
Then, we show in \Cref{lem:secondorder} that the error due to higher order terms is less than $\epsilon/2$ due to our choice of $B_0$.
We then analyze our first order error term, and show that this is sufficiently small when the termination condition is met. 
Finally, we bound the number of rounds required until the termination condition is met.

\subsection{Algorithm terminates correctly}
We begin by showing that the confidence intervals on the mean estimators hold.
To do this we define the good event $\xi$, where our confidence intervals on our mean estimators hold.
\begin{equation}
    \xi := \bigcap_{i\in[n],r\in \N} \left\{|\hat{\mu}_i^{(r)} - \mu_i| < C_r \right\}
\end{equation}
This event occurs with very high probability, as formalized in the lemma below.
\begin{lem} \label{lem:conf}
The good event $\xi$ where the estimators $\{\hat{\mu}_i^{(r)}\}$ stay within their confidence intervals satisfies $\P(\xi) \ge 1-\delta/3$.
\end{lem}
\begin{proof}
Analyzing $\xi$, we have that the probability that any of our confidence intervals fail is at most
\begin{align*}
    \P(\bar{\xi})
    =\P\left( \exists i \in [n], r \in \N, \text{ s.t. } |\hat{\mu}_i^{(r)} - \mu_i| \ge C_r\right)
    \le \sum_{i \in [n], r \in \N} 2\exp \left( - \log (12nr^2/\delta)\right)
    \le \delta/3. \numberthis
\end{align*}

\end{proof}
The rest of the analysis will be conditioned on this good event $\xi$.
Analyzing the gradient estimators $\{\hat{g}_i^{(r,L)},\hat{g}_i^{(r,U)}\}$, we have for all $i,r$ that
\begin{align}
    g_i + 2L \sqrt{n}C_r \ge \hat{g}_i^{(r,U)} \ge g_i \ge \hat{g}_i^{(r,L)}\ge g_i - 2L \sqrt{n}C_r
\end{align}
This is because, conditioned on $\xi$,
\begin{align*}
     \hat{g}_i^{(r,L)} &=  \min_{\b{y} : \|\b{y} - \bxh^{(r)}\|_\infty \le C_r} \left| \nabla_i f(\b{y})\right|\\
     &\ge g_i - \max_{\b{y} : \|\b{y} - \bx\|_\infty \le 2C_r}\left| \nabla_i f(\b{y}) - \nabla_i f(\bx) \right|\\
     &\ge g_i - \max_{\b{y} : \|\b{y} - \bx\|_\infty \le 2C_r}\left\|  \nabla f(\b{y}) - \nabla f(\bx) \right\|\\
     &\ge g_i - 2L\sqrt{n}C_r. \numberthis
\end{align*}
Additionally, $g_i \ge \hat{g}_i^{(r,L)}$ as the set we are minimizing over includes $\b{y}=\bx$. 
The result for $\hat{g}_i^{(r,U)}$ follows identically.

With Lemma \ref{lem:conf} in hand, we can now proceed with the analysis knowing that our gradient estimators stay within their confidence intervals for the duration of the algorithm with high probability.
This result directly implies the accuracy of our sampling frequencies $\alpha_i^{(r)}$, where we can see that after defining our problem complexity measure $H:= \sum_i g_i = \|\nabla f(\bx)\|_1$ we have that
\begin{align*}
    \alpha_i^{(r)}=\frac{\hat{g}_i^{(r,L)}}{\sum_j \hat{g}_j^{(r,L)}} 
    \ge \begin{cases}
    \frac{g_i}{2H} & \text{ if } g_i \ge 4L\sqrt{n}C_r,\\
    0 & \text{ otherwise.}
    \end{cases}\numberthis
\end{align*}

We now analyze the higher order error terms in our output $f(\bxt)$ caused by not knowing the true gradient.

\begin{lem} \label{lem:secondorder}
If each arm is sampled at least $2nL\log(6n/\delta)\epsilon^{-1}$ times, then with probability at least $1-\delta/3$
\begin{equation*}
    \frac{L}{2}\|\bxt - \bx\|_2^2 \le \epsilon/2
\end{equation*}
\end{lem}
\begin{proof}

We have by Hoeffding's inequality \cite{wainwright2019high} that if arm $i$ has been pulled $\tilde{T}_i^{(r)}$ times then
\begin{align*}
    \P\left(\frac{L}{2}\|\bxt - \bx\|_2^2 \ge \epsilon/2\right)
    \le \P\left(\|\bxt - \bx\|_\infty^2 \ge \frac{\epsilon}{nL}\right) 
    \le 2\sum_{i=1}^n \exp\left(- \frac{\tilde{T}_i^{(r)}\epsilon}{2nL}\right) 
    \le \delta/3, \numberthis
\end{align*}
where the last inequality holds for $\tilde{T}_i^{(r)}\ge 2nL\log(6n/\delta)\epsilon^{-1}$ which is satisfied by $B_0 \ge 2n^2L\log(6n/\delta)\epsilon^{-1}$.
\end{proof}

This requirement is non-adaptive, being independent of $\bx$, where above this sampling threshold these higher order effects can be readily bounded.

Moving to the error in our linear approximation in \eqref{eq:genFnormal}, we can design our ideal termination condition by appealing to the linear case, noting that by Hoeffding's inequality
\begin{equation}
    \sum_{i=1}^n \frac{g_i^2}{\tilde{T}_i^{(r)}} \le \frac{\epsilon^2}{8\log(6/\delta)} \quad \implies \quad \P\left( |\nabla f(\bx)^\top (\bxt-\bx)| > \epsilon/2 \right) \le \delta/3.\label{eq:genfTermCond}
\end{equation} 

Without knowledge of the true $\nabla f(\bx)$ however, we cannot compute this stopping condition.
This necessitates an algorithmic stopping condition involving our estimators $\hat{g}_i^{(r,L)}$.
Conditioning on the good event $\xi$, we proceed with our analysis in this section assuming that our gradient bounds are well behaved.
This leads to the stopping condition in \Cref{alg:generalf} of
\begin{align*}
    &\sum_{i=1}^n \frac{(\hat{g}_i^{(r,U)})^2}{\tilde{T}_i^{(r)}} \le \frac{\epsilon^2}{8\log(6/\delta)}\quad  \implies \quad \sum_{i=1}^n \frac{g_i^2}{\tilde{T}_i^{(r)}} \le \frac{\epsilon^2}{8\log(6/\delta)}, \numberthis \label{eq:genfTermCondTrue}
\end{align*}
where the implication holds on $\xi$, meaning that our algorithm will not incorrectly terminate early except on this error event.

\subsection{Bounding the sample complexity}

We analyze the number of pulls required for the bound in \eqref{eq:genfTermCondTrue} to hold by splitting our arms into 2 sets; those with $g_i$ larger or smaller than $\tau=4L\sqrt{n}C_r$.   
Note that $\tilde{B}_r \ge B_r/n$. 
Conditioned on $\xi$ we have that
\begin{align*}
     \sum_{i=1}^n \frac{(\hat{g}_i^{(r,U)})^2}{\tilde{T}_i^{(r)}}
     &\underset{(a)}{\le}\sum_{i=1}^n \frac{(g_i + 2L\sqrt{n}C_r)^2}{\tilde{T}_i^{(r)}}\\
     &\le \sum_{i:g_i<\tau}\frac{(g_i + 2L\sqrt{n}C_r)^2}{ B_r/n } + \sum_{i: g_i\ge \tau} \frac{(g_i + 2L\sqrt{n}C_r )^2}{\alpha_i B_r}\\
     &\le \sum_{i:g_i<\tau} \frac{36nL^2C_r^2}{B_r/n} + \sum_{i: g_i\ge \tau} \frac{9g_i H}{4B_r }\\
     &\le \frac{36n^3L^2C_r^2}{B_r} + \frac{9H^2}{4B_r}, \numberthis \label{eq:terminationAnalysis}
\end{align*}
where (a) utilizes the event $\xi$.

Analyzing the first term in \eqref{eq:terminationAnalysis}, we see that this is less than $\frac{\epsilon^2}{16\log(6/\delta)}$ for all $r>0$ as long as $B_0\ge 17n^{2}L\log(12n/\delta) \epsilon^{-1}$, since
\begin{align*}
    \frac{36n^3L^2C_r^2}{B_r}
    \le \frac{36n^3L^2C_1^2}{2B_0}
    \le \frac{18n^4L^2\log(12n/\delta)}{B_0^2 }
    \le \frac{\epsilon^2}{16\log(6/\delta)}. \numberthis
\end{align*}

Moving to the second term  in \eqref{eq:terminationAnalysis}, we see that it is less than $\frac{\epsilon^2}{16\log(6/\delta)}$ when $B_r \ge 36 H^2 \log(6/\delta) \epsilon^{-2}$.
This means that our algorithm will terminate no later than the first round $r$ such that
\begin{equation}
    B_r \ge 36 H^2 \log(6/\delta) \epsilon^{-2}. \label{eq:generalfBudgetReq}
\end{equation}
Due to the doubling of budgets round to round, the round where this condition is met must satisfy $B_r \le 72 H^2 \log(6/\delta) \epsilon^{-2}$.

Analyzing the total sample complexity, we see that in round $r$ Algorithm \ref{alg:generalf} takes at most $B_r+n = B_02^r+n$ samples, yielding a cumulative sample complexity of at most $2B_r-B_0+rn$ samples by the end of round $r$. 
After the termination condition is met, the algorithm will perform $\sum_{i=1}^n \tilde{T}_i^{(r)} \le 2B_r+n$ additional pulls before returning its estimate. Hence, if our algorithm terminates at the end of round $r$ its sample complexity can be upper bounded by $4B_r -B_0+ (r+1)n\le 5B_r$, as $B_0\ge n$ and $B_r \ge 2^r n \ge r n$ for $r\ge 1$.

Analyzing the failure probability of \Cref{alg:adaConf}, we see that we will output an estimate $f(\bxt)$ with error at most $\epsilon$ with probability at least $1-\delta$, as by \Cref{lem:conf},  Lemma \ref{lem:secondorder}, and \eqref{eq:genfTermCond} we have that 
\begin{align*}
    \P\left(|f(\bxt)-f(\bx)|\ge \epsilon \right)
    &\le \P\left(|f(\bxt)-f(\bx)|\ge \epsilon \ | \ \xi\right) + \P(\bar{\xi})\\
    &\le \P\left( |\nabla f(\bxt)^\top (\bxt - \bx)| \ge \epsilon/2 \ | \ \xi \right)
    +\P\left( \frac{L}{2} \|\bxt-\bx\|_2^2 \ge \epsilon/2 \ \middle|\ \xi\right) + \P(\bar{\xi})\\
    &\le \delta. \numberthis
\end{align*}

This means that Algorithm 1 will with probability at least $1-\delta$ return an estimate $f(\bxt)$ of $f(\bx)$ such that $|f(\bxt)-f(\bx)|\le \epsilon$ using at most
\begin{equation} \label{eq:thm1Result}
    \frac{85n^2L \log(12n/\delta)}{\epsilon}+ \frac{360H^2 \log(6/\delta)}{\epsilon^2}
\end{equation}
samples, as desired.

\paragraph{More refined termination conditions:}
Note that our termination condition essentially only requires that we provide an upper bound on the first order error.
In certain cases, we can utilize the structure of the function to provide tighter, but more computationally intensive, stopping conditions, where on $\xi$ we have that
\begin{equation}
    \sum_{i=1}^n \frac{g_i^2}{\tilde{T}_i^{(r)}}
    \le \max_{\b{y}: \|\b{y} - \bxh^{(r)}\|_\infty \le C_r} \sum_{i=1}^n \frac{|\nabla_i f(\b{y})|^2}{\tilde{T}_i^{(r)}}
    \le \sum_{i=1}^n \frac{\max_{\b{y}: \|\b{y} - \bxh^{(r)}\|_\infty \le C_r} |\nabla_i f(\b{y})|^2}{\tilde{T}_i^{(r)}}
    \le \sum_{i=1}^n \frac{(|\nabla_i f(\bxh^{(r)})| + L\sqrt{n}C_r)^2}{\tilde{T}_i^{(r)}}.
\end{equation}
Currently, our algorithms utilize the rightmost quantity in their stopping conditions, but the middle two can provide better performance depending on the structure of the function.
Concretely, the second from right quantity can be thought of as computing ``pessimistic'' estimators for the magnitudes of the gradient entries, where each gradient entry is assumed to be as large as it could possibly be within its confidence interval.
This contrasts with our sampling frequency construction, where for the $\ba^{(r)}$ we utilize ``optimistic'' estimators of our gradient magnitudes, where they are assumed to be as small as they could possibly be.
Note that a similar argument can be made for the construction of sampling frequencies, taking
\begin{equation}
    \alpha_i^{(r)} = \min_{\b{y}: \|\b{y} - \bxh^{(r)}\|_\infty} \frac{\nabla_i f(\b{y})}{\|\nabla f(\b{y})\|_1}.
\end{equation}
This duality between our sampling and termination frequencies allows our algorithm to efficiently run and guarantee an $\epsilon$ correct output.

\subsection{Proof of \Cref{cor:l2main}, tightened results for $f(\bx)=\|\bx\|_2^2$} \label{app:l2tighter}

In the case where the function $f$ is separable, i.e. $f(\bx)= \sum_{i=1}^n f_i(\mu_i)$, where each $f_i$ has $L$-Lipschitz derivatives, we are able to give a tighter bound on $\hat{g}_i^{(r,L)}$ as $\hat{g}_i^{(r,L)} \ge g_i -2L C_r$, since the error in estimating $\nabla_i f(\bx)$ depends only on the error in estimating $\mu_i$.
This allows us to loosen the termination condition in Line \ref{line:alg1TermCond} of \Cref{alg:generalf}, and provide a tighter bound on the lower order sample complexity terms.
If in addition $f_i(\mu_i) = \mu_i^2$ (more generally, quadratic), then we are able to partially correct for the second order error, as our plug-in estimator yields $f_i(\hat{\mu}_i) = \mu_i^2 + 2\mu_i W_i/\sqrt{\tilde{B}_r} + W_i^2/\tilde{B}_r$ for standard normal $W_i$.
While the crossterm has zero mean and involves the unknown $\mu_i$, our standard first order error term, the final term $W_i^2/\tilde{B}_r$ has expectation $1/\tilde{B}_r$ which does not depend on any unknown parameters, and can be subtracted to yield an \textit{unbiased} estimator of $f(\bx)$.
Dropping the tilde and overloading $T_i$ as $T_i=\tilde{T}_i^{(r)}$ for notational simplicity, we see that estimating $f_i(\mu_i)$ as $f_i(\tilde{\mu}_i^{(r)}) - \frac{1}{T_i}$ allows us to attain a tighter bound, as stated in \Cref{cor:l2main}.

\begin{proof}[Proof of \Cref{cor:l2main}]
Expanding the error of our approximation for our modified algorithm, we have that
\begin{align*}
    \P\left(\left|\|\bxt\|_2^2-\sum_{i=1}^n \frac{1}{T_i}-\|\bx\|_2^2\right|\ge \epsilon \right)
    &=\P\left(\left|\sum_{i=1}^n \left(\mu_i + W_i/\sqrt{T_i}\right)^2 -\frac{1}{T_i}-\sum_{i=1}^n \mu_i^2\right|\ge \epsilon \right)\\
    &=\P\left(\left|\sum_{i=1}^n \frac{2\mu_iW_i}{\sqrt{T_i}} +\sum_{i=1}^n \frac{1}{T_i} (W_i^2-1)\right|\ge \epsilon \right)\\
    &\le \P\left(\left|\sum_{i=1}^n \frac{2\mu_iW_i}{\sqrt{T_i}} \right|\ge \epsilon/2 \right)
    + \P\left(\left|\sum_{i=1}^n \frac{1}{T_i} (W_i^2-1)\right|\ge \epsilon/2 \right).\numberthis
\end{align*}
For each $i$, $\frac{1}{T_i} (W_i^2-1)$ is independent across $i$, has 0 mean, and is $\left(\frac{2}{T_i},\frac{4}{T_i}\right)$-sub-exponential \cite{wainwright2019high}. 
This shows that the sum is $(\nu_{*},b_{*})$-sub-exponential, where $\nu_{*}= 2\sqrt{\sum_{i=1}^n T_i^{-2}}$ and $b_{*} = \frac{4}{\min_i T_i}$.
Hence,
\begin{equation} \label{eq:l2secondOrder}
    \P\left( \left| \sum_{i=1}^n \frac{Z_i^2}{T_i} - \frac{1}{T_i}\right| > \epsilon/2 \right)
    \le \begin{cases}
    2\exp\left(-\frac{\epsilon^2}{8\sum_{i=1}^nT_i^{-2}} \right) & \text{ if } 0\le \epsilon \le \frac{\nu_{*}}{b_{*}}\\
    2\exp\left(-\frac{\epsilon \min_i T_i}{8} \right) & \text{ if }  \epsilon > \frac{\nu_{*}}{b_{*}}\\
    \end{cases}
\end{equation}

In order to guarantee that our higher order error terms are less than $\epsilon/2$ with probability at least $1-\delta/3$, we tackle the two possible cases.
In the first case $\sum_i T_i^{-2} \le 4n^3/B_r^2\le n^3/B_0^2$ and so $B_0 \ge \sqrt{8}n^{3/2} \log(6/\delta)\epsilon^{-1}$ is sufficient.
In the second case, $\min_i T_i \ge B_r/n\ge 2B_0/n$, and so $B_0 \ge 4n \log(6/\delta)\epsilon^{-1}$ is sufficient.
Thus, taking $B_0= 4n^{3/2}\log(6/\delta)\epsilon^{-1}$ guarantees that our error due to these second order terms in \eqref{eq:l2secondOrder} is less than $\epsilon/2$ with probability at least $1-\delta/3$.

With our tighter bound on $g_i$ due to the separability of $f$ across coordinates, we have that the first term in \eqref{eq:terminationAnalysis}, due to indices with $g_i < \tau$, is upper bounded by $36 n^2 L^2 C_r^2/B_r$.
This can be made sufficiently small by taking $B_0\ge 17n^{3/2}L\log(12n/\delta)\epsilon^{-1}$.

Combining these improved bounds on $B_0$ with the analysis from \Cref{thm:main} yields a reduced overall sample complexity of 
\begin{equation}
    \frac{85n^{3/2}L \log(12n/\delta)}{\epsilon}+ \frac{360H^2 \log(6/\delta)}{\epsilon^2}
\end{equation}
analogously to \eqref{eq:thm1Result}.
\end{proof}

\section{Proofs for \Cref{alg:thresholding}: thresholding} \label{app:thresh}

\textbf{Proof sketch:} the proof proceeds similarly to that of \Cref{thm:main}, with some interesting nuances stemming from the fact that our $\epsilon$ is now unknown.
Concretely, since we do not a priori know what accuracy to estimate $f(\bx)$ to, we need to sample according to our estimated optimal frequencies $\alpha_i^{(r)}$ in every round to generate $f(\bxt^{(r)})$ to estimate the necessary accuracy as $|f(\bxt^{(r)}) - \tau|$. 
We begin by defining the event $\tilde{\xi}$ where our estimators $f(\bxt^{(r)})$ fall within their $C_r^f$ confidence intervals, and show in \Cref{lem:thresh_conf} that this occurs with probability at least $1-2\delta/3$.
Next, we show that on $\xi$ and $\tilde{\xi}$, our algorithm will terminate correctly.
Then, since our estimators $f(\bxt^{(r)})$ are correct up to an error of $C_r^f$, we show that when $C_r^f < |f(\bx)-\tau|/2$, the termination condition must be met.
Finally, we bound the number of samples required for this to occur.

\subsection{Algorithm terminates correctly}
As before we condition on the good event $\xi$, where our mean estimators fall within their confidence intervals.
We define an additional good event $\tilde{\xi}$ where our estimators $f(\bxt^{(r)})$ fall within their confidence intervals as
\begin{equation}
    \tilde{\xi} :=  \bigcap_{r \in \N} \left\{ |f(\bxt^{(r)}) - f(\bx)| < C^f_r \right\}.
\end{equation}
With this, we can state the following Lemma.
\begin{lem} \label{lem:thresh_conf}
Following \Cref{alg:thresholding}, 
we have that
\begin{equation*}
    \P \left( \tilde{\xi} \ \middle| \ \xi  \right) \ge 1-2\delta/3
\end{equation*}
where $\xi$ is the event that our confidence intervals on the estimators $\hat{\mu}_i^{(r)}$ hold.
\end{lem}
\begin{proof}
Due to the $L$-Lipschitz gradients of $f$ we have that
\begin{equation}
    |f(\bxt)-f(\bx)| \le |\nabla f(\bx)^\top (\bxt - \bx)| + \frac{L}{2}\|\bxt-\bx\|_2^2,
\end{equation}
and so we can upper bound the deviation probability as
\begin{align*}
    \P&\left(|f(\bxt^{(r)})-f(\bx)|\ge C^f_r \ \middle|\  \xi \right)\\
    &\le \P\left(|\nabla f(\bx)^\top (\bxt^{(r)} - \bx)| + \frac{L}{2}\|\bxt^{(r)}-\bx\|_2^2\ge \sqrt{2\log(24r^2/\delta)\sum_i \frac{g_i^2}{\tilde{T}_i^{(r)}}} + C^f_{r,2} \ \middle|\ \xi\right)\\
    &\le \P\left( |\nabla f(\bx)^\top (\bxt^{(r)} - \bx)| \ge \sqrt{2\log(24r^2/\delta)\sum_i \frac{g_i^2}{\tilde{T}_i^{(r)}}} \ \middle|\ \xi \right) + \P\left( \frac{L}{2}\|\bxt^{(r)}-\bx\|_2^2 \ge C^f_{r,2} \ \middle|\ \xi \right)\\
    &\le 2\P\left( |\nabla f(\bx)^\top (\bxt^{(r)} - \bx)| \ge \sqrt{2\log(24r^2/\delta)\sum_i \frac{g_i^2}{\tilde{T}_i^{(r)}}} \right) + 2\P\left( \frac{L}{2}\|\bxt^{(r)}-\bx\|_2^2 \ge C^f_{r,2} \right). \numberthis \label{eq:threshterm1}
\end{align*}
where we used the fact that on $\xi$ we have that $\hat{g}_i^{(r,U)}\ge g_i$, allowing us to bound the width of $C_{r,1}^f$ in terms of deterministic quantities, the $g_i$ (recalling the definitions of these constants in \eqref{eq:cThresh}).
We additionally noted that for an event $E$, $\P(E|\xi) = \frac{\P(E)-\P(E|\bar{\xi}) \P(\bar{\xi})}{\P(\xi)} \le 2\P(E)$, as $\P(\xi)\ge 2/3$.

Analyzing the first term in \eqref{eq:threshterm1}, we see by Hoeffding's inequality \cite{wainwright2019high} that
\begin{align*}
    \P\left( |\nabla f(\bx)^\top (\bxt^{(r)} - \bx)| \ge \sqrt{2\log(24r^2/\delta)\sum_i \frac{g_i^2}{\tilde{T}_i^{(r)}}} \right)
    \le 2\exp\left(-\log(24r^2/\delta) \right)
    \le \frac{\delta}{12r^2}, \numberthis \label{eq:33}
\end{align*}
as $\nabla f(\bx)^\top (\bxt^{(r)} - \bx) \sim \CN\left(0,\sum_{i=1}^n \frac{g_i^2}{\tilde{T}_i^{(r)}}\right)$.
Analyzing the second term in \eqref{eq:threshterm1}, we have by Hoeffding's inequality that
\begin{align*}
    \P\left( \frac{L}{2}\|\bxt^{(r)}-\bx\|_2^2 \ge C_{r,2}^f\right)
    &\le \frac{\delta}{12r^2}. \numberthis \label{eq:34}
\end{align*}
Plugging \eqref{eq:33} and \eqref{eq:34} into \eqref{eq:threshterm1}, we have that
\begin{equation}
    \P\left(|f(\bxt^{(r)})-f(\bx)|\ge C^f_r \ \middle|\  \xi \right)\le \frac{\delta}{3r^2},
\end{equation}
and so
\begin{align*}
    \P\left(\tilde{\xi} \ \middle| \ \xi \right) 
    &= 1-\P \left(\bigcup_{r \in \N} \left\{ |f(\bxt^{(r)}) - f(\bx)| < C^f_r \right\} \ \middle| \ \xi \right)\\
    &\ge 1-\sum_{r\ge 1} \frac{\delta}{3r^2}\\
    &\ge 1-2\delta/3. \numberthis
\end{align*}
\end{proof}

Hence, on the events $\tilde{\xi},\xi$, our estimator $f(\bxt^{(r)})$ is within its confidence interval for all $r$.
This implies that on these events, \Cref{alg:thresholding} well never incorrectly output whether $f(\bxt)>\tau$:
assume for the sake of contradiction that $\xi,\tilde{\xi}$ occur, and that without loss of generality $f(\bx)>\tau$ but the algorithm outputs $f(\bxt)<\tau$.
Then for all $r$ we have
\begin{equation}
    f\left(\bxt^{(r)}\right)-\tau  
    \ge f(\bx) - \tau - C_r^f
    >-C_r^f.
\end{equation}
However, the algorithm will only terminate and output $f(\bxt)<\tau$ if $f(\bxt^{(r)}) - \tau < -C_r^f$, a contradiction.
Hence \Cref{alg:thresholding} will not output an incorrect decision on $\xi,\tilde{\xi}$.
Bounding this probability, we have that
\begin{equation} \label{eq:threshErrorProb}
    \P(\xi \cap \tilde{\xi}) = \P(\xi) \P(\tilde{\xi} \ |\  \xi) \ge (1-\delta/3)(1-2\delta/3) \ge 1-\delta.
\end{equation}

\subsection{Bounding the sample complexity}
We see that this algorithm's sample complexity will not be too large since it will stop sampling when $2C_r^f< |f(\bx)-\tau|$, as then we have that 
\begin{equation}
    |f(\bxt^{(r)}) - \tau| \ge |f(\bx)-\tau| - |f(\bxt^{(r)} - f(\bx)|\ge C_r^f,
\end{equation}
using the fact that $|f(\bxt^{(r)} - f(\bx)|\le C_r^f$ on $\xi,\tilde{\xi}$.
We can now provide our proof regarding \Cref{alg:thresholding}.

\begin{proof}[Proof of \Cref{thm:thresholding}]
\Cref{alg:thresholding} returns the correct answer with probability at least $1-\delta$, as shown in \eqref{eq:threshErrorProb}.

To provide the desired sample complexity guarantee, we bound the number of rounds required until $C_r^f < |f(\bx)-\tau|/2$ on $\xi,\tilde{\xi}$.
This is accomplished similarly to before;
\begin{align*}
    C_r^f &=\sqrt{2\log(24r^2/\delta)\sum_i \frac{\left(\hat{g}_i^{(r,U)}\right)^2}{\tilde{T}_i^{(r)}}}  + C^f_{r,2}\\
    &\underset{(a)}{\le} \sqrt{2\log(24r^2/\delta)\sum_i \frac{\left(g_i + 2L\sqrt{n}C_r\right)^2}{\tilde{T}_i^{(r)}}} + \frac{Ln \log (24nr^2/\delta)}{\tilde{B}_r}\\
    &\underset{(b)}{\le} \sqrt{\frac{2\log(24r^2/\delta) \left(36n^3L^2C_r^2+9H^2/4\right)}{B_r} } +  \frac{Ln^2 \log (24nr^2/\delta)}{B_r}\\
    &\le \sqrt{\frac{9H^2\log(24r^2/\delta)}{2B_r}} + \sqrt{\frac{144 n^3L^2\log^2(24r^2/\delta)}{B_r^2} } +  \frac{Ln^2 \log (24nr^2/\delta)}{B_r}\numberthis
\end{align*} 
where in (a) we utilized the event $\xi$, and in (b) we divided the sum into two two sets based on whether $g_i$ dominates the confidence interval or vice versa.
This must be less than $|f(\bx)-\tau|/2$ when
\begin{equation}
    B_r \ge \frac{18\log(24r^2/\delta) H^2}{(f(\bx)-\tau)^2} + \frac{24Ln^2 \log (24nr^2/\delta)}{|f(\bx)-\tau|}, \label{eq:threshPreTilde}
\end{equation}
and so our algorithm will terminate in the first round $r$ where this condition is met.
Since $B_r =n2^r$, we have that $\log(r) = \log \log (B_r/n)$, yielding that our algorithm terminates when
\begin{equation}
    B_r = \tilde{O}\left( \frac{H^2\log(1/\delta)}{(f(\bx)-\tau)^2} + \frac{n^2L\log(n/\delta)}{|f(\bx)-\tau|}\right),
\end{equation}
where $\tilde{O}$ suppresses $\log \log (\cdot)$ terms in poly$\left(n,\delta,H, (f(\bx)-\tau)^{-1} \right)$, which can be explicitly obtained by solving in \eqref{eq:threshPreTilde}.
Since \Cref{alg:thresholding} makes at most $3B_r+n$ arm pulls per round, we have that its total budget through round $r$ can be at most $\sum_{\ell=1}^r (3B_\ell+n) \le 6B_r + nr \le 7B_r$.
This means that the total number of samples used will be
\begin{equation}
    \tilde{O}\left( \frac{H^2\log(1/\delta)}{(f(\bx)-\tau)^2} + \frac{n^2L\log(n/\delta)}{|f(\bx)-\tau|}\right),
\end{equation}
yielding the desired result.
\end{proof}

\section{Proofs for \Cref{alg:adaConf}: variance adaptive algorithm}\label{app:adavar}

%%%%%% Super secret comment
% \red{can we just use function that's cts, a.e. differentiable, that's locally upper and lower bounded by a quadratic everywhere?
% Estimators 
% \begin{equation}
%     \hat{g}_i = \min_{y,z\in \text{conf}(x_i)} \left|\frac{f(y)-f(z)}{y-z}\right|
% \end{equation}
% Or
% \begin{equation}
%     \hat{g}_i := \min_{\b{y},\b{z}\in \text{Rect}(\bxh,C(T(t)))} \left| e_i^\top \frac{f(\b{y})-f(\b{z})}{\b{y}-\b{z}}\right|
% \end{equation}
% }

\textbf{Proof sketch:} we begin by stating concentration results on sample variances and sample means of bounded random variables in Lemmas \ref{lem:bernvar} and \ref{lem:bernmean}.
We then redefine the good event $\xi$ for this algorithm as the event where $\{\hat{\mu}_i^{(r)}\}$ and $\{\hat{\sigma}_i^{(r)}\}$ stay within their confidence intervals.
We show in \Cref{lem:confxi} that this occurs with probability at least $1-\delta/2$.
These immediately yields control of our gradient estimators and bounds $\{\hat{g}_i^{(r,L)},\hat{g}_i^{(r,U)}\}$ and the sampling frequencies $\{\beta_i^{(r)}\}$ and $\{\alpha_i^{(r)}\}$.
In \Cref{lem:adaConfSecond} we show that the second order error incurred by Algorithm 3 is at most $\epsilon/2$, and in \Cref{lem:adaConffirst} bound the first order error by $\epsilon/2$ when the termination condition is met.
These are combined in \Cref{lem:adaConfError} to show that the output has error at most $\epsilon$.
Lemmas \ref{lem:adaConf_bernstein_lower} and \ref{lem:adaConfBudget} guarantee that once the round budget is sufficiently large, the termination condition must be met.
Finally, we invert this requirement to yield a bound on the total number of pulls required by our algorithm.

\subsection{Algorithm terminates correctly}
We begin by controlling the sample variance $\hat{\sigma}_i^{(r)}$ of each of the $\hat{\mu}_i^{(r)}$ estimators.
For this, we turn to a concentration result regarding the sample variance.
\begin{lem}[Concentration of $\hat{\sigma}_i^{(r)}$, Theorem 10 \cite{maurer2009empirical}] \label{lem:bernvar}
Let $n\ge2$ and $\b{X} = (X_1,\hdots,X_n)$ be a vector of independent random variables with values in [0,1]. Then for $\delta'>0$ we have, writing $\E V_n$ for $\E_{\b{X}} V_n(\b{X})$,
\begin{equation*}
    \P \left(\left|\sqrt{V_n(\b{X})} - \sqrt{\E V_n}\right| > \sqrt{\frac{2\log (2/\delta')}{n-1}}\right) \le \delta,
\end{equation*}
where the sample variance $V_n(\b{X}) := \frac{1}{n(n-1)} \sum_{i<j} (X_i-X_j)^2$.
\end{lem}
This gives us concentration for the sample standard deviation estimators $\{\hat{\sigma}_i^{(r)}\}$, for which we have defined our confidence intervals as
\begin{equation}
    C_\sigma(i,r) := \sqrt{\frac{2\log (8nr^2/\delta)}{T_i^{(r)}-1}}.
\end{equation}
Assuming these confidence intervals hold, we then have that 
\begin{equation}
    \sigma_i + 2C_\sigma(i,r)\ge \hat{\sigma}_i^{(r,U)} \ge \sigma_i \ge \hat{\sigma}_i^{(r,L)}\ge \sigma_i - 2C_\sigma(i,r)    
\end{equation}
for all $i,r$ with probability at least $1-\delta/3$.

In the following Lemma we obtain sample variance dependent confidence intervals for our mean estimators $\{\hat{\mu}_i^{(r)}\}$.
\begin{lem}[Concentration of $\hat{\mu}_i^{(r)}$, Theorem 4 \cite{maurer2009empirical}]\label{lem:bernmean}
Let $Z,Z_1,\hdots,Z_n$ be i.i.d. random variables with values in [0,1], and let $\delta'>0$. Then with probability at least $1-\delta'$ we have
\begin{equation*}
    \left|\E Z - \frac{1}{n}\sum_{i=1}^n Z_i\right| \le \sqrt{\frac{2 V_n(\b{Z}) \log (4/\delta')}{n}} + \frac{7 \log (4/\delta')}{3(n-1)}
\end{equation*}
\end{lem}
Using Lemmas \ref{lem:bernvar} and \ref{lem:bernmean} we construct anytime confidence intervals for our mean estimators $\{\hat{\mu}_i^{(r)}\}$. 
In the $r$-th round, the confidence intervals are set to individually fail with probability at most $\delta'=\delta/(8nr^2)$, so that when union bounding over all $i,r$, we have that with probability at least $1-\delta$ our estimators $\{\hat{\mu}_i^{(r)}\}$  all fall within their $C_\mu(i,r)$ confidence intervals, where
\begin{equation}
    C_\mu(i,r) := \hat{\sigma}_i^{(r)}\sqrt{\frac{2 \log (32nr^2/\delta)}{T_i^{(r)}}} + \frac{7 \log (32nr^2/\delta)}{3(T_i^{(r)}-1)}.
\end{equation}

With these definitions in hand we can state the following good event $\xi$, a modification from \Cref{thm:main},  where all our confidence intervals hold:
\begin{equation}
\xi := \hspace{-.2cm}\bigcap_{i\in [n], r \in \N} \hspace{-.1cm}\left\{ \left\{|\hat{\mu}_i^{(r)} - \mu_i| \le C_\mu(i,r)\right\} \cap \left\{|\hat{\sigma}_i^{(r)} - \sigma_i| \le C_\sigma(i,r) \right\}\right\}.
\end{equation}
This event occurs with very high probability, as formalized in the lemma below.
\begin{lem}\label{lem:confxi}
The good event $\xi$ where the mean and standard deviation estimators stay within their confidence intervals satisfies $\P(\xi) \ge 1-\delta/2$.
\end{lem}
\begin{proof}
The proof follows directly from Lemmas \ref{lem:bernvar} and \ref{lem:bernmean} and a union bound, with
\begin{equation}
    \P\left(\bar{\xi}\right) \le \sum_{i,r}\P\left( \left\{|\hat{\mu}_i^{(r)} - \mu_i| > C_\mu(i,r)\right\} \cup \left\{|\hat{\sigma}_i^{(r)} - \sigma_i| 
    \le C_\sigma(i,r)\right\}\right) 
    \le \sum_{i,r} \frac{\delta}{4nr^2}
    \le \frac{\delta}{2}.
\end{equation}
\end{proof}

With this in hand, we can control the bounds of our gradient estimators $\{\hat{g}_i^{(r,L)},\hat{g}_i^{(r,U)}\}$ as before, where on $\xi$ we have that for all $i,r$
\begin{equation}
    g_i + 2L C_\mu(i,r) \ge \hat{g}_i^{(r,U)} \ge g_i \ge \hat{g}_i^{(r,L)}\ge g_i - 2L C_\mu(i,r).
\end{equation}
This follows as
\begin{align*}
     \hat{g}_i^{(r,L)} &=  \min_{y : |y - \hat{\mu}_i^{(r)}| \le C_\mu(i,r)} \left| f_i'(y)\right|\\
     &\ge g_i - \max_{y : |y - x_i| \le C_\mu(i,r)}\left| f_i'(y) - f_i'(x_i) \right|\\
     &\ge g_i - 2LC_\mu(i,r) \numberthis
\end{align*}
We additionally have that $g_i \ge \hat{g}_i^{(r,L)}$, as on $\xi$ the set we are minimizing over includes $y=\mu_i$. 

These estimators are critical in the construction of our sampling frequencies $\b{\alpha}^{(r)},\b{\beta}^{(r)}$.
On this event $\xi$, we see that our $\beta_i^{(r)}$, which impact $T_i^{(r)}$, satisfy
\begin{align*}
    \beta_i^{(r)}=\frac{\hat{\sigma}_i^{(r,L)}}{\sum_j \hat{\sigma}_j^{(r,L)}} 
    \ge \begin{cases}
    \frac{\sigma_i}{2 \sum_j \sigma_j} & \text{if } \sigma_i \ge 2C_\sigma(i,r),\\
    0 & \text{otherwise.}
    \end{cases}\numberthis
\end{align*}
Regarding the frequencies $\b{\alpha}^{(r)}$, we see that
\begin{align*}
\alpha_i^{(r)}=\frac{\hat{g}_i^{(r,L)} \hat{\sigma}_i^{(r,L)}}{\sum_j \hat{g}_j^{(r,L)} \hat{\sigma}_j^{(r,L)}} + \frac{1}{n}
\ge \begin{cases}
\frac{g_i\sigma_i}{2 \sum_j g_j\sigma_j} & \text{if } \sigma_i \ge 2C_\sigma(i,r) \text{ and } g_i \ge 4 LC_\mu(i,r),\\
1/n & \text{otherwise.}
\end{cases}\numberthis
\end{align*}

We now turn to bounding the error in our estimator $f(\bxt)$, which we can write as
\begin{equation}\label{eq:adaConferror}
    |f(\bxt)-f(\bx)| \le \left|\nabla f(\bx)^\top (\bxt - \bx)\right| + \frac{L}{2}\|\bxt-\bx\|_2^2.
\end{equation}
We begin by bounding the second order error in \eqref{eq:adaConferror}, showing that our initial sampling is sufficient to guarantee that it is less than $\epsilon/2$ with probability at least $1-\delta/4$.

\begin{lem} \label{lem:adaConfSecond}
After sampling as in \Cref{alg:adaConf} we have that 
\begin{equation*}
    \frac{L}{2}\|\bxt-\bx\|_2^2 \le \frac{\epsilon}{2}
\end{equation*}
with probability at least $1-\delta/4$.
\end{lem}
\begin{proof}
Since our arm pulls are 1-sub-Gaussian, by Hoeffding's inequality we have that 
\begin{equation}
    \P \left(|\tilde{\mu}_i - \mu_i|^2 > \frac{\epsilon}{nL}\right) \le 2\exp\left(-\frac{\tilde{T}_i^{(r)}  \epsilon}{2nL} \right).
\end{equation}
Since $\tilde{T}_i^{(r)} \ge \lceil B_0/n\rceil$, and by construction $B_0 = \frac{2n^2L \log(8n/\delta)}{\epsilon}$, we have that
\begin{align*}
    \P\left(\|\bxt-\bx\|_2^2 \le \epsilon/L\right)
    \le \P\left( \bigcup_i \left\{|\tilde{\mu}_i - \mu_i|^2 > \frac{\epsilon}{nL} \right\}\right)
    \le 2\sum_i \exp\left(-\frac{\tilde{T}_i^{(r)}  \epsilon}{2nL} \right)
    \le \delta/4.\numberthis
\end{align*}
\end{proof}

We next bound the error in our linear approximation, the first order error term in \eqref{eq:adaConferror}, showing that when our termination condition is met the error in our linear term is at most $\epsilon/2$.
\begin{lem} \label{lem:adaConffirst}
When sampling according to $\tilde{T}_i^{(r)}$ which satisfy
\begin{equation*}
    \sum_{i=1}^n \frac{g_i^2\sigma_i^2}{\tilde{T}_i^{(r)}} + \frac{\epsilon}{3}\max_i \frac{g_i}{\tilde{T}_i^{(r)}}\le \frac{\epsilon^2}{8\log(8/\delta)},
\end{equation*}
we have that
\begin{equation*}
    \P\left(\left|\nabla f(\bx)^\top (\bxt-\bx)\right| \ge \epsilon/2\right) \le \delta/4.
\end{equation*}
\end{lem}
\begin{proof}

To prove this lemma we utilize Bernstein's inequality, which states that if $X_1,\hdots,X_n$ are independent zero-mean random variables, where $|X_i|\le M$ almost surely for all $i$, then for all $\epsilon>0$
\begin{equation}
    \P\left(\left|\sum_{i=1}^n X_i\right| \ge t \right) \le 2\exp\left( - \frac{\epsilon^2/2}{\sum_{i=1}^n \E\{X_i^2\} + M\epsilon/3}\right).
\end{equation}

Denoting the $j$-th pull of arm $i$ as $Z_{i,j} \underset{\text{i.i.d.}}{\sim} Z_i$, where $Z_i$ is the distribution of arm $i$, we have for all $i,j$ that $\E\{Z_{i,j}\}=\mu_i$, $\Var(Z_{i,j}) = \sigma_i^2$, and that $|Z_{i,j}|\le 1$ almost surely.
Since $\E \{\nabla f(\bx)^\top (\bxt-\bx)\} = 0$ we can apply Bernstein's inequality on these $\sum_{i=1}^n \tilde{T}_i^{(r)}$ random variables, yielding 
\begin{align*}
    \P\left(\left|\nabla f(\bx)^\top (\bxt-\bx)\right| \ge \epsilon/2\right)
    &=\P\left(\left|\sum_{i=1}^n \sum_{j=1}^{\tilde{T}_i^{(r)}} \frac{g_i}{\tilde{T}_i^{(r)}} (Z_{i,j} -\mu_i)\right| \ge \epsilon/2\right)\\
    &\le 2\exp\left( - \frac{\epsilon^2/8}{\sum_{i=1}^n \frac{g_i^2\sigma_i^2}{\tilde{T}_i^{(r)}} + \frac{\epsilon}{3}\max_i \frac{g_i}{\tilde{T}_i^{(r)}}}\right). \numberthis
\end{align*}
As claimed, this is less than $\delta/4$ when
\begin{equation}
    \sum_{i=1}^n \frac{g_i^2\sigma_i^2}{\tilde{T}_i^{(r)}} + \frac{\epsilon}{3}\max_i \frac{g_i}{\tilde{T}_i^{(r)}} \le \frac{\epsilon^2}{8\log(8/\delta)}.
\end{equation}
\end{proof}

Combining together Lemmas \ref{lem:adaConfSecond} and \ref{lem:adaConffirst}, we have the following lemma.
\begin{lem} \label{lem:adaConfError}
If 
\begin{equation*}
    \sum_{i=1}^n \frac{g_i^2\sigma_i^2}{\tilde{T}_i^{(r)}} + \frac{\epsilon}{3}\max_i \frac{g_i}{\tilde{T}_i^{(r)}} \le \frac{\epsilon^2}{8\log(8/\delta)}
\end{equation*}
then with probability at least $1-\delta/2$
\begin{equation*}
    |f(\bxt)-f(\bx)| \le \epsilon.
\end{equation*}
\end{lem}

\subsection{Bounding the sample complexity}

The drawback of \Cref{lem:adaConfError} is that it requires knowledge of the true gradient $\nabla f(\bx)$ and the noise variances.
Since we know neither $\{\sigma_i\}$ nor $\{g_i\}$, we must instead construct a termination condition in \Cref{alg:adaConf} using our empirical estimates of these quantities.
We do this by considering the worst feasible instance within our confidence intervals, as in \Cref{alg:generalf}, since we then have that conditioned on $\xi$
\begin{equation} \label{eq:termCondUB}
     \sum_{i=1}^n \frac{g_i^2\sigma_i^2}{\tilde{T}_i^{(r)}} + \frac{\epsilon}{3}\max_i \frac{g_i}{\tilde{T}_i^{(r)}} \le \sum_{i=1}^n \frac{\left(\hat{g}_i^{(r,U)} \hat{\sigma}_i^{(r,U)}\right)^2}{\tilde{T}_i^{(r)}} + \frac{\epsilon}{3}\max_i \frac{\hat{g}_i^{(r,U)}}{\tilde{T}_i^{(r)}},
\end{equation}
where this right hand side is our computable stopping condition in \Cref{alg:adaConf} which we compare with $\frac{\epsilon^2}{8\log(8/\delta)}$.

We bound the time it takes for this termination condition to be reached in the following lemma.
We begin by analyzing the number of samples required for the second term in \eqref{eq:termCondUB} to be sufficiently small.

\begin{lem} \label{lem:adaConf_bernstein_lower}
On $\xi$, \Cref{alg:adaConf} must terminate and produce $\tilde{T}_i^{(r)}$ such that 
\begin{equation*}
    \frac{\epsilon}{3}\max_i \frac{\hat{g}_i^{(r,U)}}{\tilde{T}_i^{(r)}} \le \frac{\epsilon^2}{16\log(8/\delta)}
\end{equation*}
when
\begin{equation*}
    B_r \ge \frac{13nL^{2/3}\log(32nr^2/\delta)}{\epsilon^{2/3}} + \frac{48n\max_i g_i \log(8/\delta)}{\epsilon}.
\end{equation*}
\end{lem}
\begin{proof}
We use the same techniques as in the proof of \Cref{thm:main}, bounding our expression as
\begin{align*}
    \frac{\epsilon}{3}\max_i \frac{\hat{g}_i^{(r,U)}}{\tilde{T}_i^{(r)}}
    &\underset{(a)}{\le} \frac{\epsilon}{3}\max_i \frac{g_i + 2LC_\mu(i,r)}{\tilde{T}_i^{(r)}}\\
    &\le \frac{\epsilon}{3} \cdot \max_i\frac{g_i + 2L\left(\sqrt{\frac{2 \log (32nr^2/\delta)}{B_r/n}} + \frac{7 \log (32nr^2/\delta)}{3(B_r/n-1)}\right)}{B_r/n}\\
    &\le \frac{\epsilon}{3} \left(\frac{\max_i g_i + 8L\sqrt{\frac{ \log (32nr^2/\delta)}{B_r/n}}}{B_r/n}\right), \numberthis
\end{align*}
where (a) utilizes the event $\xi$.
We see that this expression is less than $\frac{\epsilon^2}{16\log(8/\delta)}$ when
\begin{equation}
    B_r \ge \frac{13nL^{2/3}\log(32nr^2/\delta)}{\epsilon^{2/3}} + \frac{48n\max_i g_i \log(8/\delta)}{\epsilon}.
\end{equation}
\end{proof}

We now examine the number of samples required for the dominant portion of our stopping condition, the first term in \eqref{eq:termCondUB}, to be sufficiently small.
For notational simplicity we define our new measure of problem complexity $H := \sum_i g_i \sigma_i$.
\begin{lem} \label{lem:adaConfBudget}
On $\xi$, \Cref{alg:adaConf} will produce $\tilde{T}_i^{(r)}$ such that 
\begin{equation*}
    \sum_{i=1}^n \frac{\left(\hat{g}_i^{(r,U)} \hat{\sigma}_i^{(r,U)}\right)^2}{\tilde{T}_i^{(r)}} \le \frac{\epsilon^2}{16\log(8/\delta)}
\end{equation*}
when 
\begin{equation*}
    B_r = \Omega\left(\frac{H^2 \log(1/\delta)}{\epsilon^2}
    + \left[\frac{n\sqrt{\sum_i g_i^2} + nL \sqrt{\sum_{i}\sigma_i^4}}{\epsilon} 
    + \frac{\left(nL\sum_i \sigma_i\right)^{2/3}}{\epsilon^{2/3}} 
    + \frac{ \sqrt{L}n^{5/4} }{\sqrt{\epsilon}} \right]\log (nr^2/\delta)
    \right).
\end{equation*}
\end{lem}
\begin{proof}
Examining when this termination is met, we need to upper bound terms in the right hand side of \eqref{eq:termCondUB}.
We do so by analyzing four separate cases, based on the magnitudes of $g_i$ and $\sigma_i$. 
Our primary term of consideration is when both are relatively large, in which case $\alpha_i$ (dropping the implicit $r$ superscript) are order optimal sampling frequencies.
For ease of notation, we define the set of coordinates with large $g_i$ as $\CS_g := \{i : g_i \ge 4LC_\mu(i,r)\}$, and the corresponding set $\CS_\sigma := \{ i : \sigma_i \ge 2C_\sigma(i,r)\}$, where we note that for those indices $\CS := \{i : g_i \ge 4LC_\mu(i,r) \text{ and } \sigma_i \ge 2C_\sigma(i,r)\} = \CS_g \cap \CS_\sigma$, we have $\tilde{T}_i^{(r)} \ge \alpha_i B_r \ge \frac{g_i \sigma_i}{4H} B_r$ for $i\in \CS$.
We note that due to our second round of sampling, we have that $T_i^{(r)} \ge \beta_i^{(r)} B_r \ge \frac{\sigma_i}{2\sum_j \sigma_j} B_r$ for $i\in \CS_\sigma$, and that for all $i$ we have $T_i^{(r)} \ge B_r/n$ and $\tilde{T}_i^{(r)} \ge B_r/n$.
We begin by splitting our summation into the 4 terms discussed.

\begin{align*}
     &\sum_{i=1}^n \frac{\left(\hat{g}_i^{(r,U)}  \hat{\sigma}_i^{(r,U)}\right)^2}{\tilde{T}_i^{(r)}}\\
     &\le \sum_{i=1}^n \frac{(g_i + 2LC_\mu(i,r))^2\left(\sigma_i+C_\sigma(i,r)\right)^2}{\tilde{T}_i^{(r)}}\\
     &\le \sum_{i\in \CS} \frac{6g_i^2\sigma_i^2}{\tilde{T}_i^{(r)}} 
        + \sum_{\CS_g \setminus \CS} \frac{7g_i^2 C_\sigma^2(i,r)}{\tilde{T}_i^{(r)}}
        + \sum_{\CS_\sigma \setminus \CS} \frac{81L^2C_\mu^2(i,r)\sigma_i^2}{\tilde{T}_i^{(r)}}
     +\sum_{i \in \bar{\CS}_\sigma \cup \bar{\CS}_g} \frac{324L^2C_\mu^2(i,r) C_\sigma^2(i,r)}{\tilde{T}_i^{(r)}}\numberthis \label{eq:termCondEval}
\end{align*}
We now bound each of these terms sequentially.
We begin by noting that on $\xi$
\begin{align*}
     \sum_{i \in \CS}\frac{g_i^2\sigma_i^2}{\tilde{T}_i^{(r)}} &\le \frac{H^2}{4B_r}, \numberthis \label{eq:termCond1}
\end{align*}
which is our primary error term.

Our second term stems from when $g_i$ is large but $\sigma_i$ is small, where we utilize the uniform sampling based lower bound on the number of pulls. Here we see that
\begin{align*}
    \sum_{\CS_g \setminus \CS} \frac{g_i^2 C_\sigma^2(i,r)}{\tilde{T}_i^{(r)}}
    = \sum_{\CS_g \setminus \CS} \frac{2g_i^2 \log (8nr^2/\delta)}{\tilde{T}_i^{(r)} \cdot (T_i^{(r)}-1)}
    = O\left(\frac{n^2 \left(\sum_i g_i^2\right) \log (nr^2/\delta)}{B_r^2} \right). \numberthis \label{eq:termCond2}
\end{align*}
Our third term occurs when $\sigma_i$ is large but $g_i$ is small, in which case we utilize the fact that $T_i^{(r)}\ge \frac{\sigma_i}{2 \sum_j \sigma_j} B_r$ sample proportionally to $\sigma_i$. Using the fact that on $\xi$ our estimators stay within their confidence intervals, we can bound this sum as
\begin{align*}
    \sum_{\CS_\sigma \setminus \CS} \frac{L^2C_\mu^2(i,r)\sigma_i^2}{\tilde{T}_i^{(r)}}
    &\le \sum_{\CS_\sigma \setminus \CS}\frac{L^2\left(\sigma_i\sqrt{\frac{2 \log( 32nr^2/\delta)}{T_i^{(r)}}} + \frac{13 \log (32nr^2/\delta)}{3(T_i^{(r)}-1)}\right)^2\sigma_i^2}{\tilde{T}_i^{(r)}}\\
    &= O\left( \sum_{\CS_\sigma \setminus \CS}\frac{L^2\left(\frac{ \sigma_i^2\log(nr^2/\delta)}{T_i^{(r)}} + \frac{\log^2 (nr^2/\delta)}{\left(T_i^{(r)}\right)^2}\right)\sigma_i^2}{\tilde{T}_i^{(r)}}\right)\\
    &\le O\left( L^2\sum_{\CS_\sigma \setminus \CS}\left(\frac{ \sigma_i^4\log(nr^2/\delta)}{T_i^{(r)}\cdot \tilde{T}_i^{(r)}} + \frac{\sigma_i^2\log^2 (nr^2/\delta)}{\left(T_i^{(r)}\right)^2\cdot \tilde{T}_i^{(r)}}\right)\right)\\
    &\le O\left( \frac{ n^2L^2 \left(\sum_{i}\sigma_i^4\right)\log(nr^2/\delta)}{B_r^2} + \frac{n^2 L^2\left(\sum_i \sigma_i\right)^2\log^2 (nr^2/\delta)}{B_r^3}\right). \numberthis \label{eq:termCond3}
\end{align*}
Our final term is when both $\sigma_i$ and $g_i$ are small. In this case $\sigma_i^2 \le 2C_\sigma(i,r) = \frac{2\log(8nr^2/\delta)}{T_i^{(r)}-1}$, and so we can simplify the sum as
\begin{align*}
    \sum_{i \in \bar{\CS}_\sigma \cup \bar{\CS}_g} \frac{L^2C_\mu^2(i,r) C_\sigma^2(i,r)}{\tilde{T}_i^{(r)}}
    &\le \sum_{i \in \bar{\CS}_\sigma \cup \bar{\CS}_g} \frac{ L^2 \left(\sigma_i\sqrt{\frac{2 \log( 32nr^2/\delta)}{T_i^{(r)}}} + \frac{13 \log (32nr^2/\delta)}{3(T_i^{(r)}-1)}\right)^2 \frac{2\log (8nr^2/\delta)}{T_i^{(r)}-1}}{\tilde{T}_i^{(r)}}\\
    &= O\left(\sum_{i \in \bar{\CS}_\sigma \cup \bar{\CS}_g} \frac{ L^2 \log^3 (nr^2/\delta) }{\tilde{T}_i^{(r)} \left(T_i^{(r)}\right)^3} \right)\\
    &\le O\left( \frac{ L^2n^5 \log^3 (nr^2/\delta) }{B_r^4} \right). \numberthis \label{eq:termCond4}
\end{align*}

We now see that the left hand side of our termination condition is upper bounded by plugging \eqref{eq:termCond1}, \eqref{eq:termCond2}, \eqref{eq:termCond3}, and \eqref{eq:termCond4} into \eqref{eq:termCondEval}, yielding
\begin{align*}
    &\sum_{i=1}^n \frac{\left(\hat{g}_i^{(r,U)} \hat{\sigma}_i^{(r,U)}\right)^2}{\tilde{T}_i^{(r)}}= \numberthis\\
    &O\left( \frac{H^2}{B_r} 
    + \frac{n^2\left(\sum_i g_i^2\right) \log (nr^2/\delta)}{B_r^2} 
    + \frac{ n^2L^2 \left(\sum_{i}\sigma_i^4\right)\log(nr^2/\delta)}{B_r^2} 
    + \frac{n^2 L^2\left(\sum_i \sigma_i\right)^2\log^2 (nr^2/\delta)}{B_r^3} 
    + \frac{ L^2n^5 \log^3 (nr^2/\delta) }{B_r^4}\right).
\end{align*}
We can terminate as soon as this quantity is less than $\frac{\epsilon^2}{16 \log(8/\delta)}$, which must happen when
\begin{equation}
    B_r = \Omega\left( \frac{H^2 \log(1/\delta)}{\epsilon^2} 
    + \left[\frac{n\sqrt{\left(\sum_i g_i^2\right)}}{\epsilon} 
    + \frac{n L \sqrt{\left(\sum_{i}\sigma_i^4\right)}}{\epsilon}
    + \frac{\left(nL\sum_i \sigma_i\right)^{2/3}}{\epsilon^{2/3}} 
    + \frac{ \sqrt{L}n^{5/4} }{\sqrt{\epsilon}} \right]\log (nr^2/\delta)
    \right).
\end{equation}
The final thing to note is that our algorithm starts with a budget of $B_0$, and so the total number of samples used must be at least $B_0$.
\end{proof}

With these lemmas in hand, the proof of  \Cref{thm:algbounded} is straightforward.
\begin{proof}[Proof of \Cref{thm:algbounded}]
The error of the plug-in estimator $f(\bxt)$ satisfies
\begin{align*}
    \P\left(|f(\bxt)-f(\bx)| > \epsilon\right) 
    &= \P\left(\bar{\xi}\right) + \P\left(|f(\bxt)-f(\bx)| > \epsilon \ \middle\vert \ \xi\right) \\
    &\le \P\left(\bar{\xi}\right) + 
    \P\left(\left|\nabla f(\bx)^\top (\bxt - \bx)\right| > \epsilon/2 \ \middle\vert \ \xi\right) 
    + \P\left(\frac{L}{2}\|\bxt-\bx\|_2^2 > \epsilon/2 \ \middle\vert \ \xi\right) \\
    &\le \delta, \numberthis
\end{align*}
by utilizing Lemmas \ref{lem:confxi} and \ref{lem:adaConfError}, as on $\xi$ \Cref{alg:adaConf} will not terminate until the conditions of \Cref{lem:adaConfError} are met.

Analyzing the number of arm pulls required for this algorithm to succeed, we have by Lemmas \ref{lem:adaConf_bernstein_lower} and \ref{lem:adaConfBudget} that the algorithm will terminate when
\begin{equation} \hspace{-.03cm}
    B_r = \Omega\left( \frac{H^2 \log(1/\delta)}{\epsilon^2} 
    + \left[\frac{n\sqrt{\left(\sum_i g_i^2\right)}}{\epsilon} 
    + \frac{n L \sqrt{\left(\sum_{i}\sigma_i^4\right)}}{\epsilon}
    + \frac{\left(nL\sum_i \sigma_i\right)^{2/3} + nL^{2/3}}{\epsilon^{2/3}} 
    + \frac{ \sqrt{L}n^{5/4} }{\sqrt{\epsilon}} \right]\log (nr^2/\delta)
    \right).\hspace{-.1cm}
\end{equation}
Since our algorithm makes at most $3B_\ell + 2n$ pulls during round $\ell$, and at most $4B_r+n$ pulls to construct $\bxt$ if the termination condition is met in round $r$, it in total makes no more than $10B_r$ pulls before terminating at the end of round $r$.

Converting this to an overall budget, we see that the $\log(r)$ term generates additional $\log\log$ factors due to the doubling round budgets, as in \Cref{thm:thresholding}.
This yields an overall budget of
\begin{align*}
    &O\left( \frac{H^2 \log(1/\delta)}{\epsilon^2}
    + \frac{n^2L \log(n/\delta)}{\epsilon}\right)\\
    &+ \tilde{O}\left(\left[\frac{n\sqrt{\left(\sum_i g_i^2\right)}}{\epsilon} 
    + \frac{n L \sqrt{\left(\sum_{i}\sigma_i^4\right)}}{\epsilon}
    + \frac{\left(nL\sum_i \sigma_i\right)^{2/3} + nL^{2/3}}{\epsilon^{2/3}} 
    + \frac{ \sqrt{L}n^{5/4} }{\sqrt{\epsilon}} \right]\log (n/\delta)
    \right),\numberthis
\end{align*}
where $\tilde{O}$ suppresses $\log \log (\cdot)$ terms in poly$\left(n,\delta, \sum_i g_i ,\sum_i \sigma_i , \epsilon^{-1} \right)$.
\end{proof}

\subsection{Improved $\alpha_i^{(r)}$}
As previously mentioned, the frequencies $\b{\alpha}$ in \Cref{alg:adaConf} are deisgned with simplicity and asymptotic optimality in mind. By taking
\begin{equation}
    \alpha_i^{(r)}  = \frac{\hat{g}_i^{(r,L)} \hat{\sigma}_i^{(r,L)}}{\sum_j \hat{g}_j^{(r,L)} \hat{\sigma}_j^{(r,L)}} 
    + \frac{\hat{g}_i^{(r,L)} }{\sum_j \hat{g}_j^{(r,L)} } 
    + \frac{ \left(\hat{\sigma}_i^{(r,L)}\right)^2}{\sum_j \left(\hat{\sigma}_j^{(r,L)}\right)^2} 
    + \frac{1}{n},
\end{equation}
we are able to give tighter bounds on our lower order terms, eliminating dependencies on $\sum_i \sigma_i^4$ and $\sum_i g_i^2$.

Concretely, the sample complexity bound can be improved to
\begin{align*}
    &O\left( \frac{H^2 \log(1/\delta)}{\epsilon^2}
    + \frac{n^2L \log(n/\delta)}{\epsilon}\right)\\
    &+ \tilde{O}\left(\left[\frac{\sqrt{n}\left(\sum_i g_i\right) + \sqrt{n}L \left(\sum_{i}\sigma_i^2\right)}{\epsilon} 
    + \frac{\left(nL\sum_i \sigma_i\right)^{2/3} + nL^{2/3}}{\epsilon^{2/3}} 
    + \frac{ \sqrt{L}n^{5/4} }{\sqrt{\epsilon}} \right]\log (n/\delta)
    \right). \numberthis
\end{align*}
This is achieved by using that $\tilde{T}_i^{(r)} \ge \frac{g_i}{2\sum_j g_j}$ for large enough $g_i$ in \eqref{eq:termCond2}, and that $\tilde{T}_i^{(r)} \ge \frac{\sigma_i^2}{2\sum_j \sigma_j^2}$ for large enough $\sigma_i$ in \eqref{eq:termCond3}, as sampling proportionally to $\sigma_i^2$ minimizes the maximal confidence interval width.

This only increases the sample complexity by at most a factor of 3 as we are doubling the amount of sampling we perform in the last iteration, and performing a ceiling operation.

\subsection{Gain of variance adaptivity} 
In the case with unknown variance and gradient information, we see that our optimal sampling frequencies are $\alpha_i \propto g_i \sigma_i$, yielding a first order error of
\begin{equation}
    \CN\left(0,\frac{1}{T}\sum_{i=1}^n \frac{g_i^2\sigma_i^2}{\alpha_i}\right) \sim
    \CN\left(0,\frac{1}{T}\left(\sum_{i=1}^n g_i\sigma_i\right)^2\right).
\end{equation}
Hence, we have that 
\begin{equation}
    O\left( \left(\sum_i g_i \sigma_i \right)^2 \epsilon^{-2} \log(1/\delta)\right) + o(\epsilon^{-2})
\end{equation}
samples are sufficient for an adaptive algorithm to achieve $\epsilon$ error.
For a nonadaptive scheme, we see by a similar argument that the first order error of such a scheme is distributed as 
\begin{equation}
    \CN\left(0,\frac{n}{T}\sum_{i=1}^n g_i^2\sigma_i^2\right),
\end{equation}
requiring 
\begin{equation}
    O\left( n\left(\sum_i g_i^2 \sigma_i^2 \right) \epsilon^{-2} \log(1/\delta)\right) + o(\epsilon^{-2})
\end{equation}
samples to achieve $\epsilon$ error.
Hence, the gain afforded by adaptivity is
\begin{equation}
    O\left(\frac{n\sum_i g_i^2 \sigma_i^2}{\left(\sum_i g_i \sigma_i \right)^2}\right).
\end{equation}

\section{Lower bound} \label{app:lb}
\begin{proof}[Proof of \Cref{prop:lowerbound}]

For a plug-in estimator $\bxh$ generated by sampling the $i$-th arm $T_i$ times, we see that the error $f(\bx)-f(\bxh)$ is distributed as
\begin{equation}
\b{g}^\top (\bx-\bxh)
\sim \CN\left(0,\sum_i \frac{g_i^2\sigma_i^2}{T_i}\right)
\end{equation}
when arm pulls from arm $i$ are corrupted by independent additive Gaussian noise with variance $\sigma_i^2$.
We then have by standard Gaussian anti-concentration inequalities that $\P(\CN(0,1)> x) \ge \frac{1}{4}e^{-x^2}$, and so
\begin{equation}
    \P \left(\CN\left(0,\sum_i \frac{g_i^2\sigma_i^2}{T_i}\right) \ge \epsilon \right) \ge \frac{1}{4}\exp\left(-\frac{\epsilon^2}{\sum_i \frac{g_i^2\sigma_i^2}{T_i}} \right),
\end{equation}
where in order for this to be less than $\delta$, it must be the case that
\begin{equation}
    \P \left(\CN\left(0,\sum_i \frac{g_i^2 \sigma_i^2}{T_i}\right) \ge \epsilon \right) <\delta \implies \sum_i \frac{g_i^2\sigma_i^2}{T_i} \le \frac{\epsilon^2}{\log(1/4\delta)}.
\end{equation}
Rewriting this as  $T_i = \alpha_i T$, where $\b{\alpha}$ is the sampling probability vector, we have as before that
\begin{align*}
    \min_{\alpha \in \Delta^n} \sum_i \frac{g_i^2\sigma_i^2}{\alpha_i T} = \frac{1}{T} \left(\sum_i g_i\sigma_i\right)^2 \numberthis
\end{align*}
by taking the Lagrangian, invoking Sion's minimax theorem \cite{sion1958general}, and optimizing to yield $\alpha_i \propto g_i\sigma_i$.
Plugging this back in, we see that
\begin{align}
    \sum_i \frac{g_i^2\sigma_i^2}{\alpha_i T} \le \frac{\epsilon^2}{\log(1/4\delta)} \implies T \ge \left(\sum_i g_i\sigma_i\right)^2\epsilon^{-2} \log(1/4\delta).
\end{align}
\end{proof}

\section{Experimental Details} \label{app:expdetails}

We implemented all algorithms in Python.
The results and figures in this paper can be reproduced from the available code.
We now describe important practical optimizations for our adaptive methods to outperform the optimism-based fully sequential method and batched uniform scheme.
All experiments were run on one core of an AMD Opteron Processor 6378 with 500GB memory (no parallelism within a trial).
No empirical standard deviations were included, as all simulations were run for 100 trials for the optimism based fully sequential method and 1000 trials for the batched methods.
Due to the low variance in our outputs, the confidence intervals around our average error are extremely small.
We use $\delta= .01$ for all experiments.

\subsection{Additional numerical results} \label{app:extraexp}
For generating data for our $Ax$ experiments, we construct our ground truth $\bx$ vector as drawn from a pareto(.5) distribution, then normalize it so the largest entry is 20, to maintain a reasonable signal to noise ratio.
We then construct $A = \bx x^\top / \|x\|_2^2 + Z$, where the entries of $x$ are i.i.d. Uniform([.5,1]), and $Z\in\R^{n\times d}$ has i.i.d. normal entries with variance $1/d^2$, to prevent the variances from scaling with the dimension.
We correct the final product by updating
$A \gets A - (Ax - \bx) \mathds{1}^\top / \|x\|_1$ to ensure that $Ax = \bx$.
This functions as desired, as since $x>0$ entrywise we have that $\|x\|_1 = \mathds{1}^\top x$.
We provide several additional simulations below in \Cref{fig:supFig}, corroborating our numerical results in the main body.
Note that when there is an underlying computational problem, as in the case of estimating $\|Ax\|_2^2$, this underlying computational problem allows for exact computation of certain coordinates, resulting in much larger gains of adaptivity for certain computational budgets.

\begin{figure}[h]
    \centering
    \begin{subfigure}[b]{0.32\linewidth}
    \includegraphics[width=\linewidth]{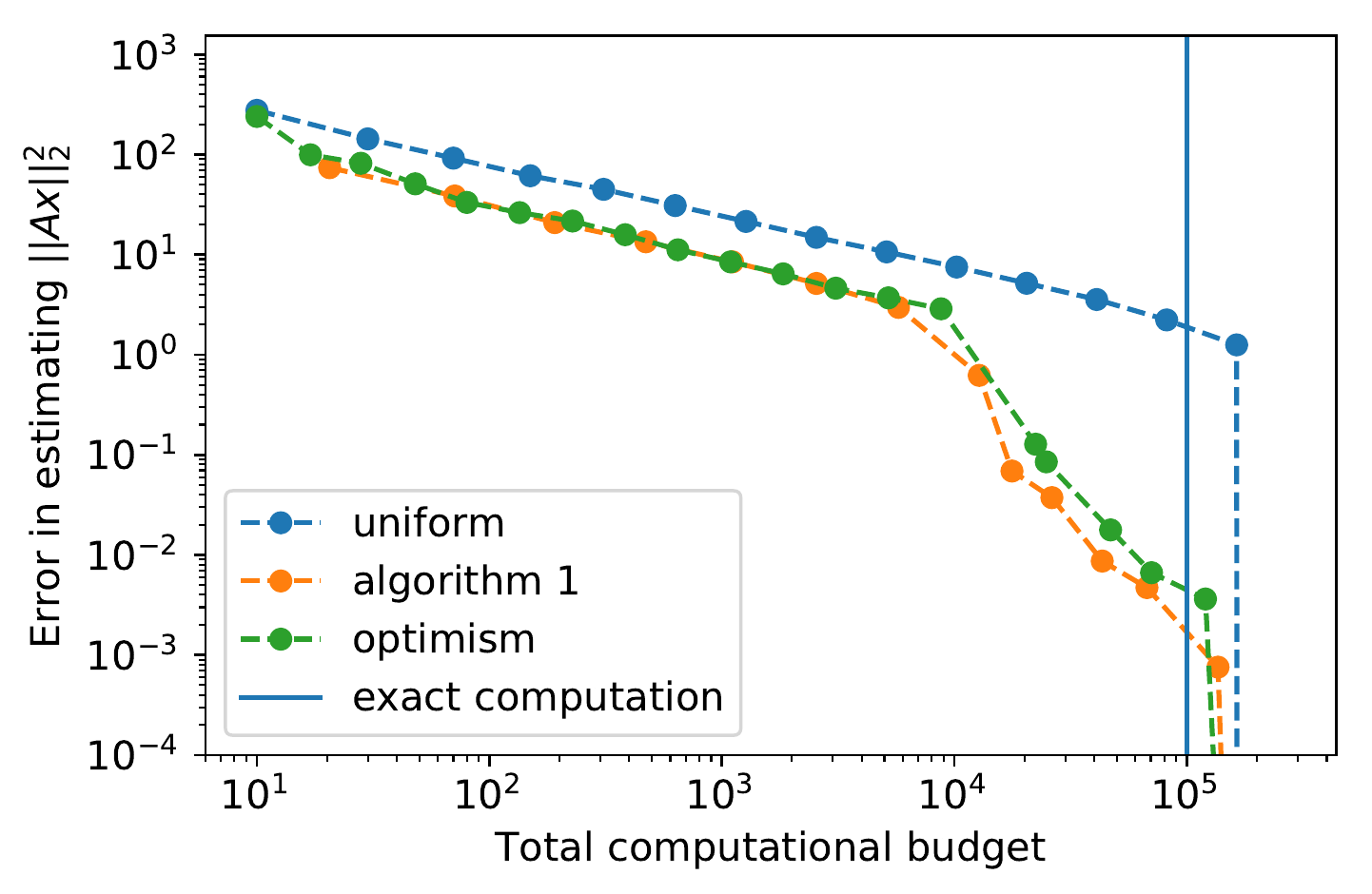}
    \subcaption{}
    \label{fig:Ax_n10_d10k_wOpt}
    \end{subfigure}
    \begin{subfigure}[b]{.32\linewidth}
    \includegraphics[width=\linewidth]{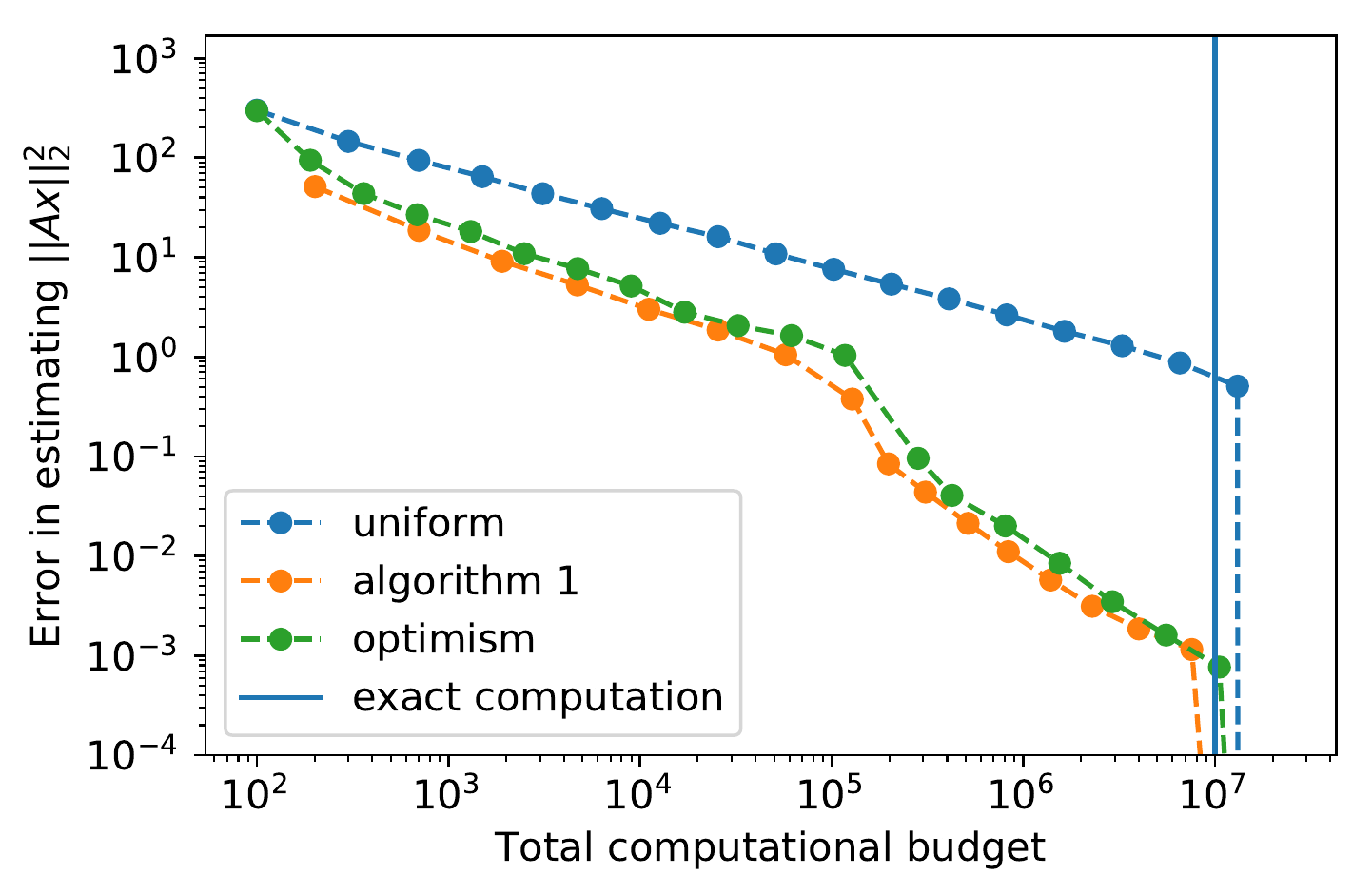}
    \subcaption{}
    \label{fig:Ax_n100_d100k_wOpt}
    \end{subfigure}
    \begin{subfigure}[b]{.32\linewidth}
    \includegraphics[width=\linewidth]{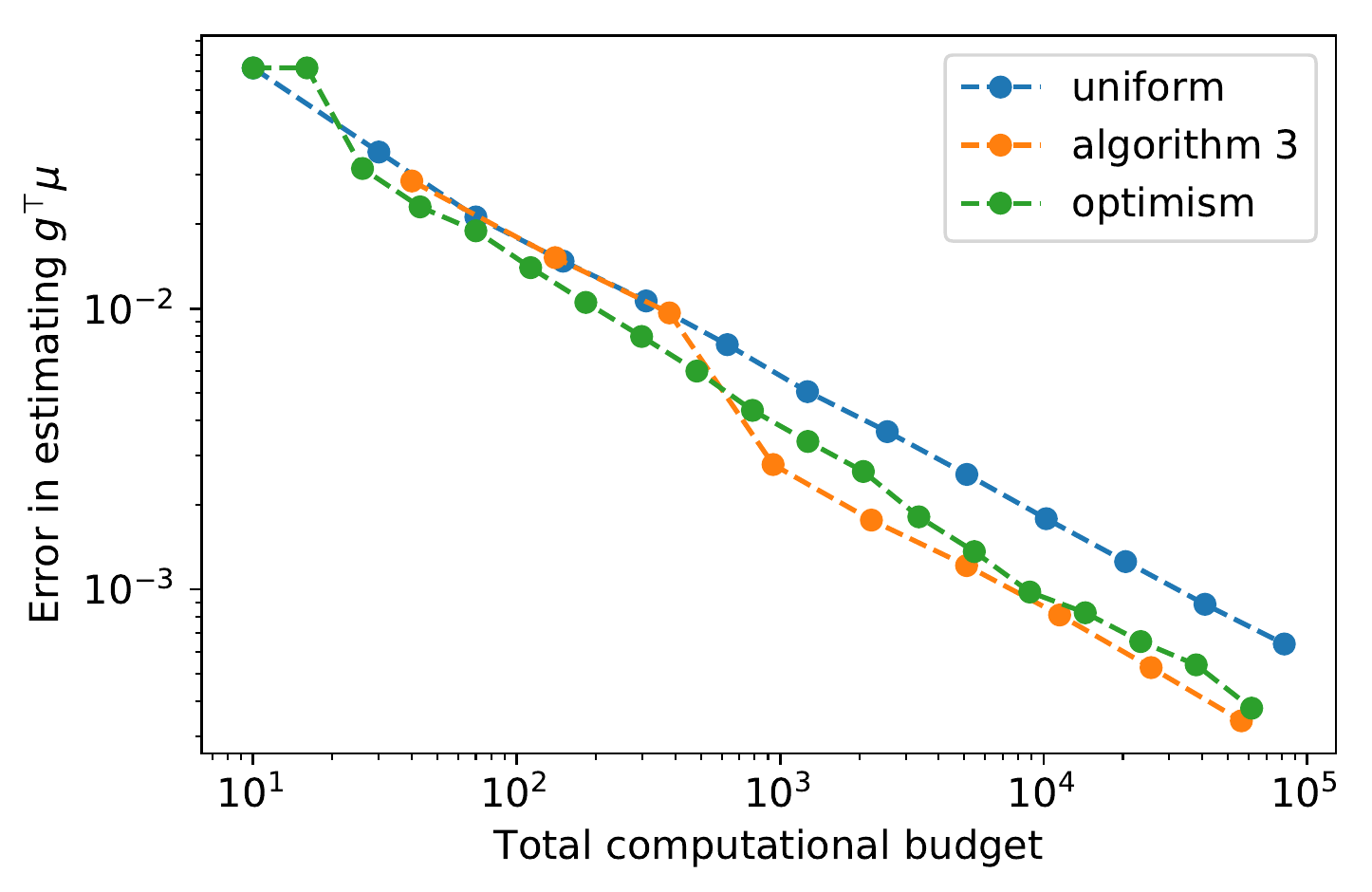}
    \subcaption{}
    \label{fig:bernoulli_n10_pareto}
    \end{subfigure}
    \caption{In (a) we have the estimation error of $\|Ax\|_2^2$ for matrix $A$, vector $x$, for $n=10$ and $d=10k$. The problem instance shown in \Cref{fig:Ax_n10_d10k_wOpt} is projected to have a gain of adaptivity of $7.5$. In (b) we see the estimation error of $\|Ax\|_2^2$ for matrix $A$, vector $x$, for $n=100$ and $d=10k$.
    In (c) we plot the estimation error of $\b{g}^\top \bx$ for $n=10$, $w$ uniform in [0,1] and $\bx$ normalized pareto distribution.} \label{fig:supFig}
\end{figure}

\subsection{Modifications from written algorithm} \label{app:changesAlg}
In order to obtain good finite sample performance, we make several changes to the algorithms as stated.
Firstly, we reuse all samples throughout the algorithm, including in the final outputted $\bxt$.
Secondly, we start from a budget of 1 pull per arm with $B_0=n$, and double from there.
Additionally,  we use as our confidence interval width $C_r = \sqrt{\log(1/\delta)/\tilde{B}_r}$.
Finally, we use $\tilde{T}_i^{(r)} = \lceil \min(r,10)\alpha_i B_r \rceil$, where uniform sampling is not practically necessary due to the reuse of samples, and the factor of 10 allows our algorithm to perform better in terms of constants.
Our theoretical analysis can accommodate this factor of 10.
Note that this factor of 10 directly translates into the slope of the linear relationship; since the algorithm will have made approximately $2B_r$ uniform samples by the end of round $r$, it is only able to allocate $10/(10+2)\approx .83$ fraction of its samples to the optimal sampling frequencies.
This is empirically validated by our simulations, as we can see that the slope of the line of best fit is indeed approximately .83.
This parameter of 10 needs to be tuned carefully in practice, as it indicates how greedy the algorithm will be; if this factor is chosen to be too large, then the algorithm will ignore higher order effects and act too strongly on the imperfect gradient information.
For the plots showing the error as a function of number of samples, we take all the uniform samples up to round $r$ for the adaptive algorithm, then generate the termination sampling pattern $\tilde{T}_i^{(r)}$, sample according to that, and measure and record the error.
These samples, corresponding to the $\tilde{T}_i^{(r)}$, are then discarded, and uniform samples for round $r+1$ are taken.

For the gain in terms of stopping time of adaptive over uniform sampling, we use slightly perturbed realizations to better display the trend, and avoid edge / rounding effects.
To this end, we modify our algorithms to utilize the round budget $B_r = n (2-c)^r$.
We then iterate over this base, ranging $c$ from $c_0=2$ to $c_{99}=\frac{301}{200}$ with $c_i=2-\frac{i}{200}$ for 100 trials.
No failures were recorded in the stopping time simulations; that is, all outputted $\bxt$ satisfied $|f(\bxt)-f(\bx)|\le \epsilon$ for both the batched uniform and our adaptive algorithms.
Additionally, the stopping conditions yielded comparable error for both methods, with average outputted error within a factor of 2 for the uniform and adaptive methods across all settings tested.

We implement the optimism-based algorithm of \cite{carpentier2015adaptive} with a priority queue to ensure efficient selection of an arm at time $t$, requiring amortized $\log n$ time instead of $O(n)$ time.
We additionally generate all possible Bernoulli samples ahead of time, and within our timed experiments only query the pre-drawn samples.
Confidence intervals are used with $\beta=\sqrt{\log(1/\delta)}$.

In order to make the algorithms efficient in the $\|Ax\|_2^2$ example, we first randomly permute the columns of $A$ and entries of $x$ jointly, and then take our arm pulls as iterating over these entries in order.
This allows the uniform arm pulls in our batched algorithms to be taken as BLAS efficient operations, simply taking the matrix vector product between a column subsampled $A$ and $x$.
Note that this can have statistical benefits from correlating the noise in our sampling, as discussed in \cite{baharav2019ultra}.

\subsection{Optimism based algorithm}
We describe our optimism based algorithm below for estimating $f(\bx)= \|\bx\|_2^2$ in the case of $\|Ax\|_2^2$.
When we refer to pulling arm $i$, we mean selecting a random coordinate $J\sim \text{Unif}([d])$ and computing $A_{i,J}x_J$.
We utilize a priority queue on the $B_{i,t}$, noting that only $B_{i^{*}_t,t}$ will change at any given time step. 
This means that every iteration of this algorithm can be performed in amortized $O(\log n)$ time, instead of the naive $O(n)$ complexity of recomputing the $B_{i,t}$ from scratch in every iteration and iterating over to find the largest.
\begin{algorithm}[h]
\begin{algorithmic}[1]
\caption{\label{alg:adaOptimismApp} \texttt{Optimism based algorithm for $\|Ax\|_2^2$}}
\State \textbf{Input:} arms $[n]$, matrix $A\in\R^{n\times d}$, vector $x \in \R^d$, target accuracy $\epsilon$, error probability $\delta$
\State Pull each arm once
\For{$t=n+1,\hdots,T$}
\State Compute for each arm $i\in [n]$
$$ B_{i,t} = \frac{1}{T_{i,t}} \left(\hat{\mu}_{i,t} + \frac{\sqrt{\log(1/\delta)}}{T_{i,t}}\right)$$
\State Select $i^{*}_t \in \argmax_{i: T_{i,t}\le d} B_{i,t}$
\If{ $T_{i,t}=d$}
\State Exactly compute $\mu_i = A_i^\top x$
\Else
\State Pull arm $i^{*}_t$
\EndIf
\If{$T_{i,t} = d$ for all $i$} 
\State \Return{Exactly computed $\|\bxh_t\|_2^2 = \|Ax\|_2^2$}
\EndIf
\EndFor
\State \Return $\sum_{i=1}^n \hat{\mu}_{i,t}^2 - \frac{\mathds{1}\{ T_{i,t} <d\}}{T_{i,t}} $
\end{algorithmic}
\end{algorithm}

\end{appendix}

\end{document}

%% file: tdefs.tex
\usepackage{amsfonts, dsfont, amssymb, amsthm, amsmath,graphicx,xcolor,algorithm}

\newcommand{\ignore}[1]{}

\definecolor{forestgreen}{rgb}{0.0, 0.27, 0.13}

\newcommand\numberthis{\addtocounter{equation}{1}\tag{\theequation}}

\usepackage{semantic}
\mathlig{==}{\equiv}
\mathlig{=.}{\doteq}
\mathlig{:=}{\triangleq}
\mathlig{<<}{\ll}
\mathlig{>>}{\gg}
\mathlig{<>}{\neq}
\mathlig{<=}{\leq}
\mathlig{>=}{\geq}
\mathlig{<==}{\Leftarrow}
\mathlig{==>}{\Rightarrow}
\mathlig{<=>}{\Leftrightarrow}
\mathlig{<==>}{\iff}
\mathlig{<--}{\leftarrow}
\mathlig{-->}{\rightarrow}
\mathlig{<->}{\leftrightarrow}
\mathlig{+-}{\pm}
\mathlig{-+}{\mp}
\mathlig{...}{\dots}
\mathlig{!=}{\stackrel{!}{=}}

\DeclareMathOperator*{\argmax}{\arg\!\max}
\DeclareMathOperator*{\argmin}{\arg\!\min}
\newif\ifshowanswer    % by default set to false.
\newcommand{\isitthree}[1]
{
  \ifnum#1=3
    number #1 is 3
  \else
    number #1 is not 3
  \fi
}

\newcommand{\Var}{\mathrm{Var}}

\newcommand{\be}{\begin{equation}}
\newcommand{\ee}{\end{equation}}

\newcommand\R{{\mathbb{R}}}

\renewcommand\P{{\mathds{P}}}
\newcommand\E{{\mathds{E}}}

\renewcommand\b{\boldsymbol}

\newcommand\CN{{\mathcal N}}

\newcommand\CS{{\mathcal S}}

\newcommand\CX{{\mathcal X}}

\newcommand\N{{\mathbb N}}

%% file: main_new.bbl
\begin{thebibliography}{}

\bibitem[Antos et~al., 2010]{antos2010active}
Antos, A., Grover, V., and Szepesv{\'a}ri, C. (2010).
\newblock Active learning in heteroscedastic noise.
\newblock {\em Theoretical Computer Science}, 411(29-30):2712--2728.

\bibitem[Bagaria et~al., 2021]{bagaria2021bandit}
Bagaria, V., Baharav, T.~Z., Kamath, G.~M., and David, N.~T. (2021).
\newblock Bandit-based monte carlo optimization for nearest neighbors.
\newblock {\em IEEE Journal on Selected Areas in Information Theory}.

\bibitem[Bagaria et~al., 2018]{bagaria2018medoids}
Bagaria, V., Kamath, G., Ntranos, V., Zhang, M., and Tse, D. (2018).
\newblock Medoids in almost-linear time via multi-armed bandits.
\newblock In {\em International Conference on Artificial Intelligence and
  Statistics}, pages 500--509. PMLR.

\bibitem[Baharav and Tse, 2019]{baharav2019ultra}
Baharav, T. and Tse, D. (2019).
\newblock Ultra fast medoid identification via correlated sequential halving.
\newblock {\em Advances in Neural Information Processing Systems}, 32.

\bibitem[Balkanski and Singer, 2018]{balkanski2018adaptive}
Balkanski, E. and Singer, Y. (2018).
\newblock The adaptive complexity of maximizing a submodular function.
\newblock In {\em Proceedings of the 50th annual ACM SIGACT symposium on theory
  of computing}, pages 1138--1151.

\bibitem[Bubeck, 2013]{bubeckBlog}
Bubeck, S. (2013).
\newblock Oracle complexity of smooth convex functions.

\bibitem[Bubeck and Cesa-Bianchi, 2012]{bubeck2012regret}
Bubeck, S. and Cesa-Bianchi, N. (2012).
\newblock Regret analysis of stochastic and nonstochastic multi-armed bandit
  problems.
\newblock {\em Foundations and Trends in Machine Learning}, 5(1):1--122.

\bibitem[Bubeck et~al., 2013]{bubeck2013bandits}
Bubeck, S., Cesa-Bianchi, N., and Lugosi, G. (2013).
\newblock Bandits with heavy tail.
\newblock {\em IEEE Transactions on Information Theory}, 59(11):7711--7717.

\bibitem[Cai et~al., 2004]{cai2004optimal}
Cai, K.-Y., Li, Y.-C., and Liu, K. (2004).
\newblock Optimal and adaptive testing for software reliability assessment.
\newblock {\em Information and Software Technology}, 46(15):989--1000.

\bibitem[Carpentier and Munos, 2012]{carpentier2012adaptive}
Carpentier, A. and Munos, R. (2012).
\newblock Adaptive stratified sampling for monte-carlo integration of
  differentiable functions.
\newblock {\em Advances in neural information processing systems}, 25.

\bibitem[Carpentier et~al., 2015]{carpentier2015adaptive}
Carpentier, A., Munos, R., and Antos, A. (2015).
\newblock Adaptive strategy for stratified monte carlo sampling.
\newblock {\em J. Mach. Learn. Res.}, 16:2231--2271.

\bibitem[Catoni, 2012]{catoni2012challenging}
Catoni, O. (2012).
\newblock Challenging the empirical mean and empirical variance: a deviation
  study.
\newblock In {\em Annales de l'IHP Probabilit{\'e}s et statistiques},
  volume~48, pages 1148--1185.

\bibitem[Drineas et~al., 2006]{drineas2006fast}
Drineas, P., Kannan, R., and Mahoney, M.~W. (2006).
\newblock Fast monte carlo algorithms for matrices i: Approximating matrix
  multiplication.
\newblock {\em SIAM Journal on Computing}, 36(1):132--157.

\bibitem[Esfandiari et~al., 2021]{esfandiari2021adaptivity}
Esfandiari, H., Karbasi, A., and Mirrokni, V. (2021).
\newblock Adaptivity in adaptive submodularity.
\newblock In {\em Conference on Learning Theory}, pages 1823--1846. PMLR.

\bibitem[Hillel et~al., 2013]{hillel2013distributed}
Hillel, E., Karnin, Z.~S., Koren, T., Lempel, R., and Somekh, O. (2013).
\newblock Distributed exploration in multi-armed bandits.
\newblock {\em Advances in Neural Information Processing Systems}, 26.

\bibitem[Hu et~al., 2013]{hu2013enhancing}
Hu, H., Jiang, C.-H., Cai, K.-Y., Wong, W.~E., and Mathur, A.~P. (2013).
\newblock Enhancing software reliability estimates using modified adaptive
  testing.
\newblock {\em Information and Software Technology}, 55(2):288--300.

\bibitem[Jamieson et~al., 2014]{jamieson2014lil}
Jamieson, K., Malloy, M., Nowak, R., and Bubeck, S. (2014).
\newblock lil’ucb: An optimal exploration algorithm for multi-armed bandits.
\newblock In {\em Conference on Learning Theory}, pages 423--439.

\bibitem[Jamieson and Nowak, 2014]{jamieson2014best}
Jamieson, K. and Nowak, R. (2014).
\newblock Best-arm identification algorithms for multi-armed bandits in the
  fixed confidence setting.
\newblock In {\em Information Sciences and Systems (CISS), 2014 48th Annual
  Conference on}, pages 1--6. IEEE.

\bibitem[Kamath et~al., 2020]{kamath2020adaptive}
Kamath, G., Baharav, T., and Shomorony, I. (2020).
\newblock Adaptive learning of rank-one models for efficient pairwise sequence
  alignment.
\newblock {\em Advances in Neural Information Processing Systems},
  33:7513--7525.

\bibitem[Karbasi et~al., 2021]{karbasi2021parallelizing}
Karbasi, A., Mirrokni, V., and Shadravan, M. (2021).
\newblock Parallelizing thompson sampling.
\newblock {\em Advances in Neural Information Processing Systems}, 34.

\bibitem[Karnin et~al., 2013]{karnin2013almost}
Karnin, Z., Koren, T., and Somekh, O. (2013).
\newblock Almost optimal exploration in multi-armed bandits.
\newblock In {\em International Conference on Machine Learning}, pages
  1238--1246. PMLR.

\bibitem[Karpov and Zhang, 2020]{karpov2020batched}
Karpov, N. and Zhang, Q. (2020).
\newblock Batched coarse ranking in multi-armed bandits.
\newblock {\em Advances in Neural Information Processing Systems}, 33.

\bibitem[Kaufmann et~al., 2016]{kaufmann2016complexity}
Kaufmann, E., Capp{\'e}, O., and Garivier, A. (2016).
\newblock On the complexity of best-arm identification in multi-armed bandit
  models.
\newblock {\em The Journal of Machine Learning Research}, 17(1):1--42.

\bibitem[Kim and Nelson, 2001]{kim2001fully}
Kim, S.-H. and Nelson, B.~L. (2001).
\newblock A fully sequential procedure for indifference-zone selection in
  simulation.
\newblock {\em ACM Transactions on Modeling and Computer Simulation (TOMACS)},
  11(3):251--273.

\bibitem[Kocsis and Szepesv{\'a}ri, 2006]{kocsis2006bandit}
Kocsis, L. and Szepesv{\'a}ri, C. (2006).
\newblock Bandit based monte-carlo planning.
\newblock In {\em European conference on machine learning}, pages 282--293.
  Springer.

\bibitem[Lai and Robbins, 1985]{lai1985asymptotically}
Lai, T.~L. and Robbins, H. (1985).
\newblock Asymptotically efficient adaptive allocation rules.
\newblock {\em Advances in applied mathematics}, 6(1):4--22.

\bibitem[Lattimore and Szepesv{\'a}ri, 2020]{lattimore2020bandit}
Lattimore, T. and Szepesv{\'a}ri, C. (2020).
\newblock {\em Bandit algorithms}.
\newblock Cambridge University Press.

\bibitem[LeJeune et~al., 2019]{lejeune2019adaptive}
LeJeune, D., Heckel, R., and Baraniuk, R. (2019).
\newblock Adaptive estimation for approximate $ k $-nearest-neighbor
  computations.
\newblock In {\em The 22nd International Conference on Artificial Intelligence
  and Statistics}, pages 3099--3107. PMLR.

\bibitem[Li et~al., 2017]{li2017hyperband}
Li, L., Jamieson, K., DeSalvo, G., Rostamizadeh, A., and Talwalkar, A. (2017).
\newblock Hyperband: A novel bandit-based approach to hyperparameter
  optimization.
\newblock {\em The Journal of Machine Learning Research}, 18(1):6765--6816.

\bibitem[Lohr, 2019]{lohr2019sampling}
Lohr, S.~L. (2019).
\newblock {\em Sampling: design and analysis}.
\newblock Chapman and Hall/CRC.

\bibitem[Lv et~al., 2014]{lv2014asymptotic}
Lv, J., Yin, B.-B., and Cai, K.-Y. (2014).
\newblock On the asymptotic behavior of adaptive testing strategy for software
  reliability assessment.
\newblock {\em IEEE transactions on Software Engineering}, 40(4):396--412.

\bibitem[Mason et~al., 2019]{mason2019learning}
Mason, B., Tripathy, A., and Nowak, R. (2019).
\newblock Learning nearest neighbor graphs from noisy distance samples.
\newblock {\em Advances in Neural Information Processing Systems}, 32.

\bibitem[Mason et~al., 2021]{mason2021nearest}
Mason, B., Tripathy, A., and Nowak, R. (2021).
\newblock Nearest neighbor search under uncertainty.
\newblock In {\em Uncertainty in Artificial Intelligence}, pages 1777--1786.
  Proceedings of Machine Learning Research.

\bibitem[Maurer and Pontil, 2009]{maurer2009empirical}
Maurer, A. and Pontil, M. (2009).
\newblock Empirical bernstein bounds and sample variance penalization.
\newblock {\em Conference on Learning Theory}, 22.

\bibitem[Paulson, 1964]{paulson1964sequential}
Paulson, E. (1964).
\newblock A sequential procedure for selecting the population with the largest
  mean from k normal populations.
\newblock {\em The Annals of Mathematical Statistics}, pages 174--180.

\bibitem[Robbins, 1952]{robbins1952some}
Robbins, H. (1952).
\newblock Some aspects of the sequential design of experiments.
\newblock {\em Bulletin of the American Mathematical Society}, 58(5):527--535.

\bibitem[Singhal et~al., 2021]{singhal2020query}
Singhal, A., Pirojiwala, S., and Karamchandani, N. (2021).
\newblock Query complexity of k-nn based mode estimation.
\newblock In {\em 2020 IEEE Information Theory Workshop (ITW)}, pages 1--5.
  IEEE.

\bibitem[Sion, 1958]{sion1958general}
Sion, M. (1958).
\newblock On general minimax theorems.
\newblock {\em Pacific Journal of mathematics}, 8(1):171--176.

\bibitem[Tiwari et~al., 2020]{tiwari2020bandit}
Tiwari, M., Zhang, M.~J., Mayclin, J., Thrun, S., Piech, C., and Shomorony, I.
  (2020).
\newblock Banditpam: Almost linear time k-medoids clustering via multi-armed
  bandits.
\newblock {\em Advances in Neural Information Processing Systems},
  33:10211--10222.

\bibitem[Wainwright, 2019]{wainwright2019high}
Wainwright, M.~J. (2019).
\newblock {\em High-dimensional statistics: A non-asymptotic viewpoint},
  volume~48.
\newblock Cambridge University Press.

\bibitem[Zhang et~al., 2019]{zhang2019adaptive}
Zhang, M., Zou, J., and Tse, D. (2019).
\newblock Adaptive monte carlo multiple testing via multi-armed bandits.
\newblock In {\em International Conference on Machine Learning}, pages
  7512--7522. PMLR.

\end{thebibliography}
